\documentclass[11pt]{article}

\usepackage{fullpage}
\usepackage[utf8]{inputenc}
\usepackage[T1]{fontenc}
\usepackage{microtype}
\usepackage{amsmath,amssymb,amsfonts,amsthm,mathtools}

\usepackage{xcolor}
\usepackage{tikz}
\usetikzlibrary{arrows.meta,positioning}
\usepackage{wrapfig}
\usepackage{dsfont}
\usepackage{enumitem}

\usepackage[authoryear,sort&compress]{natbib}
\usepackage{url}
\usepackage{hyperref}
\hypersetup{
    colorlinks=true,
    linkcolor=black,
    citecolor=black,
    urlcolor=black
}
\usepackage{cleveref}
\usepackage{authblk}
\usepackage{amartya_ltx}
\renewcommand{\epsilon}{\varepsilon}

\title{Why Forget-Only Unlearning Needs Memorization}
\author[1]{Luka Radi\'{c}}
\author[1]{Vikrant Singhal}
\author[1]{Amartya Sanyal}
\affil[1]{Department of Computer Science, University of Copenhagen}
\affil[ ]{\texttt{\{luka.radic, vikrant.singhal, amsa\}@di.ku.dk}}
\date{}

\begin{document}

\maketitle
\begin{abstract}
    Machine unlearning asks for a deletion algorithm whose
    output is close to retraining from scratch without the selected
    forget examples. In this work, we study forget-only unlearning,
    where the deletion algorithm receives only the trained model and
    the examples to forget, with no retained data or extra training
    information. We ask whether forget-only unlearning is always
    possible. We first show that this depends on the learning method:
    different datasets can produce the same trained model but require
    very different outputs after the same examples are removed. Using
    this observation, we derive lower bounds on how accurately
    unlearning can match retraining and instantiate it for several
    standard learning algorithms. We then ask what must be true when
    forget-only unlearning succeeds. To this end, we derive lower
    bounds on what an algorithm must memorize about the training data
    to handle arbitrary deletion requests. For simple threshold
    learners, the required information can be as large as the entire
    dataset, even though ordinary training keeps only one boundary
    point. Overall, our results show that information discarded during
    ordinary learning may be needed later for deletion, so models
    designed for forget-only unlearning may need to retain more
    information than standard training does.
\end{abstract}

\section{Introduction}\label{sec:intro}

As machine learning models are more widely deployed, the ability to
retroactively remove specific data from their training sets has become
important. This need arises from privacy regulations, such as the EU
GDPR's right to be forgotten~(\citet{gdpr2016}), and from the
accidental inclusion of poisoned~\citep{schoepf2024potion},
copyrighted~\citep{dou2025avoiding}, or unlawfully
collected~\citep{thiel2023identifying} data. While retraining a model
on the remaining data is the gold standard, doing so is often too
expensive or practically impossible. Certified machine
\emph{unlearning}, therefore, asks for a cheaper procedure whose
output is close in a distributional sense to the model obtained by
retraining from scratch on the retained data~\citep{cao2015towards}.

An unlearning scheme consists of a learning algorithm \(\learning\)
and an unlearning algorithm \(\unlearning\). Given a training dataset
\(S\), the learner outputs a model
\(
    M\coloneqq \learning(S).
\) 
The unlearner receives a forget set $U \subseteq S$ and the trained
model, and outputs the unlearned model
\(
    \overline{W} \coloneqq \unlearning(M, U),  
\)
with the goal of being distributionally close to the retraining target
$W \coloneqq \learning \br{S \setminus U}$.  In this paper, we study
the \emph{forget-only} setting: at deletion time, \(\unlearning\) has
access only to \(M\) and \(U\), with no retained data or other
training information except what is already stored in \(M\). This
raises a difficulty: information discarded during training could be
necessary to determine the correct retraining target after deletion.

We define unlearning using the notion of R\'enyi divergence. A pair
\((\learning,\unlearning)\) satisfies \((\alpha,\varepsilon)\)-R\'enyi
unlearning (RU)~(\Cref{def:renyi-unlearning}) if, for every $S$ and $U
\subseteq S$, it holds that
\(
d_\alpha\!\left(
\unlearning(M,U),\,
\unlearning(W,\emptyset)
\right)\leq\varepsilon,
\)
where \(d_\alpha\) denotes the symmetric \(\alpha\)-Rényi divergence.
We note that R\'enyi unlearning implies the standard
\((\epsilon',\delta)\)-unlearning for appropriate choice of
\(\epsilon'\). Similar to~\citet{sekhari2021RememberWhatYou}, we
compare the unlearned output with \(\unlearning(W,\emptyset)\) to
allow for deterministic learners. However, indistinguishability alone
does not guarantee utility: an unlearner could output the same
distribution regardless of its inputs and still satisfy RU. We
therefore also require the empty-request unlearning to preserve
utility by remaining \emph{close} to its input, with high probability (see~\Cref{def:gamma-utility}).

Several empirical unlearning methods are forget-only or close to
forget-only, including commonly used methods based on gradient ascent
on the forget set
\citep{zhang2024npo,fan2025simplicity,jang2023knowledge}. These
methods can work well in some settings, but fail in
others~\citep{mavrothalassitis2025ascent,yu2025impossibility}. This
leads to our first question:
\begin{question}
\label{q:forget-only-possible}
Is forget-only unlearning always possible? If not, what conditions
make it impossible?
\end{question}

Our first result in~\Cref{sec:impossibility} shows that forget-only
unlearning is not always possible. Intuitively, the obstruction arises
when many~(say two) datasets $S_1$, $S_2$ produce the same trained
model, $\learning(S_1) = \learning(S_2)$, but deleting the same forget
set $U$ pushes their corresponding retraining targets far apart. For
example, consider a threshold learner that returns the leftmost
positive point. Two datasets may share this point but have different
second-leftmost positive points, which become the respective
retraining targets when the shared point is deleted. A forget-only
unlearner receives the same input in both cases, namely
$(\learning(S_1),U) = (\learning(S_2),U)$ and must therefore produce
the same output distribution. However, that output cannot be close to
two far apart retraining targets if they are required to preserve
utility on empty unlearning requests. We generalize these ideas to
lower bound the unlearning parameter
$\varepsilon$ (\Cref{thm:concentration-impossibility}), and instantiate
it for several common learning algorithms
(\Cref{ex:thresholds,ex:median,ex:svm}). One consequence is the
following.

\begin{corollary}[Informal: hard-margin SVM]
Let \(\learning\) be hard-margin linear SVM in dimension \(d\geq 2\).
Then, no forget-only unlearner with non-trivial utility can satisfy
\((\alpha,\varepsilon)\)-R\'{e}nyi unlearning with a fixed finite
\(\varepsilon\) uniformly over all datasets and deletion requests.
\end{corollary}

The hardness above arises because the trained SVM has discarded
information needed after deletion. If the model does not memorize
enough information to distinguish datasets that look identical before
deletion but different after deletion, then the unlearner has no way
to recover the correct retraining target.  Additionally, we show that
such a hard instance is guaranteed to exist if the \emph{teaching
dimension}~\citep{goldman1995complexity} of the hypothesis class is
finite.  Our impossibility result leads to the second question:
\begin{question}
\label{q:memorization-needed}
How much information about the training dataset must the learned model
memorize in order to support future forget-only unlearning requests?
\end{question}

Following~\citet{brown2021memorization,feldman2025tradeoffs}, we measure memorization through the \textit{mutual information}
\(I(M;S)\), where \(S\) is drawn from an analysis distribution over datasets.
Our main contribution is a technique for lower-bounding \(I(M;S)\)
using the retraining targets induced by possible deletion requests,
given the unlearning and utility requirements. We prove the bound for
discrete model spaces in~\Cref{sec:lb-discrete}, and instantiate it for concrete
learning problems.

Since the unlearning guarantee is uniform across every dataset and
deletion request, it must hold for every \(S\) and every nonempty
request \(U_i\subseteq S\). The model \(M\) must therefore retain
enough information to support each retraining target
\(W_i\coloneqq\learning(S\setminus U_i)\).  The need for memorization
\(I(M;S)\) might seem to add across unlearning requests, but requests
can reveal overlapping information: e.g., for a threshold learner,
many deletions can expose the same next positive point. Thus the bound
must count only new, nonredundant information per request.  We
formalize this via a permutation \(\pi\) of the requests. The bound
holds for every ordering, so we can choose the one giving the
strongest lower bound.  The following informal statement summarizes
our result for finite model spaces.

\begin{theorem}
Let \(S\) be a training dataset, \(M=\learning(S)\),
\((\learning, \unlearning)\) satisfy \(\br{\alpha,\varepsilon}\)-Rényi unlearning,
and \(U_1,\ldots,U_K\) be the unlearning requests. For each request,
let \(W_i=\learning(S\setminus U_i)\) be the corresponding
retrained model. Then, for any ordering \(\pi\) of the requests,
\[
\begin{aligned}
I(M;S)
\gtrsim{}&
\sum_{i=1}^K
H\!\left(
W_{\pi(i)}
\mid W_{\pi(<i)}, U_{\pi(\leq i)}
\right)-
K\cdot\mathrm{err}(\alpha,\varepsilon).
\end{aligned}
\]
\end{theorem}

The first term captures the genuinely new information in the
respective retraining targets while the second term gives some slack
for the approximation error. We instantiate this bound for several
learning algorithms and hypothesis classes
in~\Cref{sec:examples-disc}. For canonical threshold learners, this
bound gives a concrete separation between the information retained by
ordinary learning and the information required for forget-only
unlearning. We show that over datasets of size $n$ from a domain of
size $N$, the ordinary learner keeps only one boundary point and
carries $\log(N/n)$ bits of information about the dataset. However,
supporting unlearning requests of size at most $m$, with $n\geq2m$,
requires memorizing $\Omega(m\log(N/n))$ bits of information.

Closest to our work,~\citet{cherapanamjeri2025space} lower bounds the
space complexity of learning--unlearning algorithms in terms of the
eluder dimension of the hypothesis class. Our bounds depend on the
particular learning algorithm and the retraining targets induced by
the chosen deletion requests. Additionally, their bounds are not for
forget-only algorithms and mainly apply to the pure unlearning
setting.  Taken together, our results give, to our knowledge, one of
the first formal accounts of why memorization is necessary for
forget-only unlearning. 
Thus, forget-only unlearning entails a basic tradeoff: either the
learner preserves enough information for future deletions, or the
unlearner must rely on additional retained data, training state, or
surrogate information.

\section{Preliminaries}\label{sec:preliminaries}

\noindent\textbf{Datasets and deletion requests.} Let \(\cX\) be the
data space and let \(\cW\) be the model space. A dataset of size \(n\)
is an ordered tuple \(S=\br{x_1,\ldots,x_n}\in\cX^n\). A deletion
request (or forget set) is denoted by \(U\subseteq S\), with
\(|U|\leq m\), where \(1\leq m\leq n\) is the deletion budget.
We write \(S\setminus U\) for the retained dataset. 

The unlearning guarantee holds for every dataset \(S\) and unlearning
request \(U\) and, unless stated otherwise, probabilities are taken
only over algorithmic randomness. In~\Cref{sec:memorisation}, where we
lower-bound \(I(M;S)\), we let \(S\) be random under an analysis
distribution \(\mathcal{P}_S\) over datasets. No i.i.d.  assumption is
made unless stated explicitly.

\noindent\textbf{Learning and forget-only unlearning.} A learning
algorithm is a possibly randomized map \(\learning:\cX^\star\to\cW\),
where \(\cX^\star\) is the set of finite datasets. An unlearning
algorithm is a randomized map
\(\unlearning:\cW\times\cX^\star\to\cW\).  For fixed \(S\in\cX^n\) and
\(U\subseteq S\), define \(M \coloneqq \learning(S)\), \(W \coloneqq
\learning(S\setminus U)\), \(\overline{W} \coloneqq
\unlearning(M,U)\), and \(W^{\emptyset} \coloneqq
\unlearning(W,\emptyset)\).  We call these the pretrained model,
retraining target, unlearned model, and empty-request output,
respectively.

The unlearner \(\unlearning\) is \emph{forget-only}: at deletion time,
its only inputs are the pretrained model \(M\) and the forget set
\(U\).  It has no access to the retained data, training logs,
checkpoints, gradients, or auxiliary statistics beyond any information
stored in \(M\). This differs from settings where the unlearner may
also receive second-order information or the retained
dataset~\citep{sekhari2021RememberWhatYou,neel2021descent}.

\noindent \textbf{Rényi unlearning.}
For distributions $P,Q$ and $\alpha>1$, we use the symmetric
$\alpha$-Rényi divergence
\(
    d_\alpha(P, Q) \coloneqq \max \bc{
        \renyi{\dvg{P}{Q}},
        \renyi{\dvg{Q}{P}}
    }
\), where
\(
    \renyi{\dvg{P}{Q}}
    \coloneqq
    \frac{1}{\alpha-1}
    \log
    \int
    \br{
        \frac{d P}{d Q}
    }^\alpha
    d Q,
\)
with the convention $\renyi{\dvg{P}{Q}} = \infty$ if $P$ is not
absolutely continuous with respect to $Q$.  We write $\law(X)$ for the
law of a random variable $X$ and use the shorthand
$d_\alpha(X,Y) \coloneqq d_\alpha(\law(X),\law(Y))$.

\begin{definition}[R\'enyi unlearning]
\label{def:renyi-unlearning}
Fix \(\alpha>1\) and \(\varepsilon\geq0\). We say that
\(\br{\learning,\unlearning}\) satisfies
\(\br{\alpha,\varepsilon}\)-Rényi unlearning (RU) over deletion sets of
size at most \(m\) if, for every dataset \(S\in\cX^n\) and every
request \(U\subseteq S\) with \(|U|\leq m\),
\[
    d_\alpha\br{
        \unlearning\br{\learning(S),U},
        \unlearning\br{\learning(S\setminus U),\emptyset}
    }
    \le \varepsilon .
\]
When \(m\) is clear from context, we simply say that
\(\br{\learning,\unlearning}\) satisfies
\(\br{\alpha,\varepsilon}\)-R\'enyi unlearning.
\end{definition}

The reference $W^\emptyset=\unlearning(W,\emptyset)$ is obtained by
applying unlearning the empty set to the retrained model. This allows
us to have randomized unlearning outputs even when $\learning$ is
deterministic~\citep{sekhari2021RememberWhatYou}. However, an
unlearner that always returns the same distribution would satisfy RU
regardless of its inputs. We therefore impose an empty-request utility
requirement: $\unlearning(w,\emptyset)$ should remain close to $w$ in
model space.

\begin{definition}[\(\Gamma\)-utility on empty requests]
\label{def:gamma-utility}
Fix a metric \(d_{\cW}\) on \(\cW\) and a non-increasing function
\(\Gamma:[0,\infty)\to[0,1]\). For each \(w\in\cW\), let
\(Q_w\coloneqq\law(\unlearning(w,\emptyset))\). We say that \(\unlearning\)
has \((d_{\cW},\Gamma)\)-utility on empty requests if, for every
\(w\in\cW\) and every radius \(r\geq0\),
\[
    \prob_{W^\emptyset\sim Q_w}
    \!\left[d_{\cW}(W^\emptyset,w)>r\right]\leq\Gamma(r).
\]
\end{definition}

For discrete $\cW$, we use the zero-one metric
$d_{\cW}(w,w')=\dsI\{w\neq w'\}$. At $r=0$, writing $\eta\coloneqq\Gamma(0)$,
the utility condition gives $\prob_{W^\emptyset\sim
Q_w}[W^\emptyset\neq w]\leq\eta$ for every $w\in\cW$. For a vector
space $\cW$ with norm $\phi$, we use $d_{\cW}(w,w')=\phi(w-w')$ and
write $(\phi,\Gamma)$-utility. The example below shows that the
Gaussian mechanism satisfies this utility condition.
\begin{example}[Gaussian empty-request map]
\label{ex:gaussian-utility}
Suppose \(\cW=\mathbb R^d\) and
\(\unlearning(w,\emptyset)=w+Z\), where
\(Z\sim\mathcal N(0,\sigma^2 I_d)\). Then \(\unlearning\) has
\(\br{\norm{\cdot}_2,\Gamma}\)-utility with
\(\Gamma(r)=1\) for \(r<\sigma\sqrt d\), and
\(\Gamma(r)=\exp\br{-(r-\sigma\sqrt d)^2/(2\sigma^2)}\) otherwise.
\end{example}

\section{Impossibility of Forget-Only Unlearning}\label{sec:impossibility}

In this section, we show that forget-only unlearning cannot, in
general, achieve both small \(\varepsilon\) and nontrivial utility.
Such an unconditional lower bound cannot hold for every learner: a
learner that memorizes the full dataset, e.g., by storing an encoding
of the entire training dataset in the model, can recover the retained
dataset from \((M,U)\). Any such lower bound must therefore
necessarily depend on the learner. 

Our lower bound depends on a learner-specific property, that we define
as a~\emph{shared-deletion example}.  This also connects
to~\Cref{sec:memorisation}, where we quantify the amount of
memorization required for successful unlearning.

Our construction is based on the idea that there are multiple datasets
whose differences are invisible while learning, but become clear after
deletion. Specifically, \(N\) datasets produce identical or close
distributions over pretrained models, yet after deleting the same
forget set, their retraining outputs are separated. We formalize this
condition below, with \(\Delta\) measuring the separation between the
retraining outputs.
\begin{definition}[Shared-deletion example]
\label{def:deletion-gap}
Fix a learning algorithm \(\learning\), a norm \(\phi\), integers
\(N\ge2\) and \(m\ge0\), and parameters \(\alpha>1\), \(\Delta>0\),
\(\delta_\alpha\ge0\), and \(\kappa\in[0,1]\).  We call
\((S_1,\ldots,S_N,U)\), with \(S_i\in\cX^n\), an
\((N,m,\Delta,\phi;\delta_\alpha,\kappa)\)-\emph{shared-deletion
example} (s.d.e.) for \(\learning\) if:
\begin{enumerate}[leftmargin=1.5em]
    \item \(U\subseteq\bigcap_{i=1}^N S_i\) and \(|U|\le m\).

    \item For every distinct \(i,j\in[N]\),
    \(
        d_\alpha\bigl(\learning(S_i),\learning(S_j)\bigr)
        \le \delta_\alpha.
    \)

    \item There exist measurable sets \(E_1,\ldots,E_N\subseteq\cW\)
    such that, for every distinct \(i,j\in[N]\),
    \(
        \phi(x-z)\ge\Delta\) for all \(x\in E_i,\ z\in E_j,
    \)
    and, for every \(i\in[N]\),
    \(
        \prob_{W\sim\learning(S_i\setminus U)}[W\in E_i]
        \ge 1-\kappa.
    \)
\end{enumerate}
\end{definition}

An \(N\)-sized s.d.e. consists of \(N\) distinct datasets whose
pretrained-model distributions are close, but whose retraining
distributions concentrate on well-separated sets after deleting the
same forget set. Similar packing conditions appear in several lower
bound proofs like for Differential Privacy~\citep{hardt2010geometry}.
We use this condition to prove our main impossibility result.

\begin{theorem}
\label{thm:concentration-impossibility}
Fix \(\alpha >1\). Let \(\learning,\unlearning\) be
\(\br{\alpha,\epsilon}\)-RU for deletion requests of size at most
$m$ and let \(\unlearning\) have $(\phi, \Gamma)$-utility on empty
requests. If \(\learning\) admits an \((N, m, \Delta, \phi;
\delta_\alpha, \kappa)\)-s.d.e., then it must hold that:

\noindent \textbf{Exact collision} ($\delta_\alpha = 0$):
    \[\varepsilon \geq
        \sup_{r \in R}
        \max \bc{
            \log N + \frac{1}{\gamma} \log (1 - u(r)),
            \frac{1}{\gamma} \log \frac{N-1}{N} - \log u(r)
        }.
        \]
\noindent \textbf{Approximate collision} ($\delta_\alpha > 0$):
    \[
        \varepsilon \geq
        \sup_{r \in R}
        \max \bc{
            \gamma \br{\log N - \delta_\alpha}
            + \frac{1}{\gamma} \log (1 - u(r)),
            \frac{1}{\gamma^2} \log \frac{N-1}{N} - \frac1\gamma \delta_\alpha - \log u(r)
        },
    \]
where $R \coloneqq \bc{r \in (0, \Delta/2): 0 < u(r) < 1}$, $u(r) = \kappa + \Gamma(r)$, and \(\gamma=\frac{\alpha -1}{\alpha}\).
\end{theorem}

\noindent We provide a proof sketch for the exact-collision below. The
full proof is deferred
to~\Cref{app:proof-concentration-impossibility}.
\begin{proof}[Proof sketch]

Assume first $\delta_\alpha = 0$. Write $W_i \coloneqq \learning(S_i
\setminus U)$ and $Q_i \coloneqq \law(\unlearning(W_i, \emptyset))$.
Since $\learning(S_i)$ has the same law for all $i$ and $\unlearning$
is forget-only, all $N$ deletion requests have the same output law $T
\coloneqq \law(\unlearning(M, U))$, $M \sim \law\br{\learning(S_1)}$, and the
unlearning guarantee gives $\renyi{\dvg{Q_i}{T}} \le \varepsilon$ and
$\renyi{\dvg{T}{Q_i}} \le \varepsilon$ for every $i$.

Fix $r \in (0, \Delta/2)$ with $0 < u(r) < 1$. By the separation
property of shared deletion examples (property $(3)$) and
\Cref{lem:renyi-from-concentration}, there are disjoint sets $B_i$
with $Q_i(B_i) \ge 1 - u(r)$ such that each retraining target
concentrates on its own set, yet the single law $T$ must stay close to
all of them.  Then, using the R\'enyi unlearning inequality (\Cref{lem:renyi-event-transfer}), we get
that \(T\) must give some probability mass to every \(B_i\):
\[
    1 - u(r) \le Q_i(B_i) \le e^{\gamma\varepsilon}\, T(B_i)^{\gamma}
    \implies
    T(B_i) \ge e^{-\varepsilon} \br{1 - u(r)}^{1/\gamma}.
\]
However, since $T$ has only one unit of mass to share among the disjoint
sets:
\[
    1 \ge \sum\nolimits_{i=1}^{N} T(B_i) \ge N e^{-\varepsilon} \br{1 - u(r)}^{1/\gamma}
    \implies
    \varepsilon \ge \log N + \tfrac{1}{\gamma} \log\br{1 - u(r)} .
\]
This is the first term of the maximum. The second follows from the same steps
applied to $B_i^c$ with $\renyi{\dvg{T}{Q_i}} \le \varepsilon$, which forces $T$
to stay inside every $B_i$, giving $T(B_i) \ge 1 - e^{\gamma\varepsilon} u(r)^{\gamma}$.
Taking the supremum over $r$ gives the exact collision lower bound.

For $\delta_\alpha > 0$, the unlearned laws $T_i$ may differ, but data
processing gives $d_\alpha(T_i, T_1) \le \delta_\alpha$. Each estimate then
passes through $T_i$ first; this extra application
of~\Cref{lem:renyi-event-transfer} replaces $\gamma$ by $\gamma^2$ and adds the
$\delta_\alpha$ penalties in the approximate collision lower bound.

\end{proof}

\noindent \Cref{thm:concentration-impossibility} provides a bound on
the unlearning guarantee whenever shared-deletion examples exist.  We
next show that shared deletion examples arise in simple, standard
learning problems. The examples below illustrate the multitude of
possible \(\Delta\) in shared deletion examples: a constant fraction
of the model-space diameter, the full diameter, or an unbounded gap.
The examples are stated for \(N=2\), but their constructions extend to
larger \(N\), with suitable choices of parameters. For each of these
examples, fix a dataset size \(n>2\) and a deletion budget \(1\leq m \leq n-2\). 

\begin{example}[Canonical thresholds]
\label{ex:thresholds}
Fix an integer \(c>1\). Let the data space be
\(\cX_c\times\{-1,+1\}\), where \(\cX_c=\{-c,-c+1,\ldots,c\}\). On a
realizable sample containing both labels, the canonical threshold
learner \(\learning_c^{\mathrm{th}}\) returns the midpoint between the
largest negative example and the smallest positive example.  This
learner admits a
\((2,m,c-\frac12,\abs{\cdot};0,0)\)-s.d.e.\looseness=-1
\end{example}

\begin{example}[Empirical median]
\label{ex:median}
Fix an integer \(c>0\). Let \(\learning_c^{\mathrm{med}}\) be the
empirical median learner on samples supported on \([-c,c]\).  The
largest \(\Delta\) for which this learner admits a
\((2,m,\Delta,\abs{\cdot};0,0)\)-s.d.e.  is as follows.  If \(n=2r\)
for some \(r\in\mathbb N\), then \(\Delta=c\) for \(m\in\{1,2\}\), and
\(\Delta=2c\) for \(m\ge3\).  If \(n=2r-1\) for some \(r\in\mathbb
N\), then \(\Delta=c\) for \(m=1\), and \(\Delta=2c\) for \(m\ge2\).
\end{example}

\begin{example}[Hard-margin linear classification]
\label{ex:svm}
Fix dimension \(d\ge2\). Let \(\learning_d^{\mathrm{svm}}\) be the
hard-margin homogeneous linear SVM on realizable samples supported on
\(\{x\in\reals^d:\norm{x}_2\le1\}\times\{-1,+1\}\). That is,
\(\learning_d^{\mathrm{svm}}\) returns the minimum-norm separator
\(w\) satisfying \(y\ip{w}{x}\ge1\) for all training examples
\((x,y)\). For every \(\Delta>0\), this learner admits a
\((2,1,\Delta,\norm{\cdot}_2;0,0)\)-s.d.e.

The learner returns the unnormalized SVM vector, so its model
space is unbounded and the separation \(\Delta\) can be
arbitrarily large. If the learner instead returns the normalized
direction \(w/\norm{w}_2\), the same construction gives a
\((2,1,\sqrt2,\norm{\cdot}_2;0,0)\)-s.d.e.
\end{example}

Combining these examples with~\Cref{thm:concentration-impossibility}
gives lower bounds on \(\varepsilon\) for the corresponding
separations \(\Delta\), with \(N=2\) and \(\delta_\alpha=\kappa=0\).
For the SVM learner, shared-deletion examples exist for every finite
separation \(\Delta>0\).  Under the Gaussian empty-request unlearning
from~\Cref{ex:gaussian-utility} with \(\sigma>0\), taking
\(r=\Delta/4\) in the second exact-collision bound gives
\[
    \varepsilon
    \ge
    \frac{(\Delta/4-\sigma\sqrt d)^2}{2\sigma^2}
    -\frac{\log2}{\gamma}
\]
for \(\Delta>4\sigma\sqrt d\). Since \(\Delta\) can be arbitrarily
large, this lower bound grows arbitrarily large. Thus, for this SVM
learner, no forget-only unlearner with fixed-variance Gaussian
empty-unlearning post-processing can satisfy a uniform
\((\alpha,\varepsilon)\)-R\'enyi unlearning guarantee with finite
\(\varepsilon\), even for single-example deletions.

In fact, we show that a stronger result holds: for certain hypothesis classes, any empirical risk minimizer under the zero-one loss is guaranteed to admit an s.d.e. once the dataset is sufficiently large. In particular, this applies to hypothesis classes with finite teaching dimension, i.e., the smallest number of labeled examples sufficient to uniquely identify any hypothesis in the class~\citep{goldman1995complexity}; see~\Cref{def:teaching-dim}.

\begin{proposition}
    \label{pp:teaching-dimension}
    Fix \(N>1\) and let $\cH$ be a hypothesis class with finite
    teaching dimension $\tau$ and $|\cH| \geq N$. Let $\learning$ be
    any empirical risk minimizer for the zero-one loss, $\phi$ be any
    norm on a vector space containing $\cH$, and let $h_1,\dots,h_N
    \in \cH$ be distinct.  Then there exist $(S_1,\dots,S_N,U)$ with
    $|U| = O(\tau^2)$ and $|S_i| = o(\tau^3)$ for all $i \in [N]$ such
    that $\learning(S_i) = \learning(S_j)$ and $\learning(S_i
    \setminus U) = h_i$ for all $i\neq j \in [N]$.  Consequently,
    $\learning$ admits an
    \(
        (N, O(\tau^2), \Delta, \phi; 0, 0)
    \)-s.d.e. with $\Delta \coloneqq \min_{i \neq j} \phi(h_i - h_j)$.
\end{proposition}

\section{Memorization for Forget-Only Unlearning}\label{sec:memorisation}

The previous section identified an obstruction to forget-only
unlearning: a uniform \((\alpha,\varepsilon)\)-RU guarantee is
controlled by the hardest shared-deletion instance of the learner
\(\learning\). In the examples above, two datasets share a common
subset that hides the difference between them; deleting that shared
part exposes the difference and separates the retraining targets,
which makes forget-only unlearning impossible. We now ask the
complementary question: when forget-only unlearning \textit{is}
possible, how much information about \(S\) must \(\learning(S)\) store
to support it?

Forget-only unlearning can succeed only if the retraining target is
determined by the pair \((\learning(S),U)\), for every admissible
deletion \(U\). We quantify this required information through the
mutual information \(I(\learning(S);S)\), where \(S\sim\mathcal P_S\)
is drawn from an analysis distribution over datasets. For any
collection of possible deletion requests \(\mathcal
U=\bc{U_1,\ldots,U_K}\subseteq 2^S\) we choose, the trained model must
therefore implicitly store sufficient information to distinguish the
corresponding retraining targets. Hence, the learner must memorize
at least the information that tells these targets apart, which leads
to the intuition
\[
    I(\learning(S);S)
    \geq
    I\!\left(
        \learning(S);
        \bc{\learning(S\setminus U)}_{U \in \mathcal U}
    \right).
\]

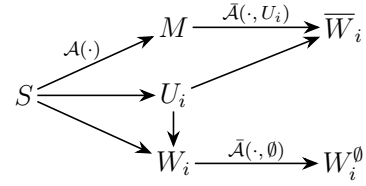
\begin{wrapfigure}{r}{0.4\textwidth}
    \centering
    \begin{tikzpicture}[
        scale=0.9,
        every node/.style={transform shape},
        >=Stealth,
        rv/.style={
            inner sep=1.5pt,
            align=center,
            font=\large
        },
        edge/.style={
            -{Stealth[length=2.2mm,width=1.7mm]},
            semithick,
            draw=black
        },
        lab/.style={
            font=\scriptsize,
            inner sep=1pt
        }
    ]

    \node[rv] (S) at (0,0) {$S$};

    \node[rv] (M) at (2.2,1) {$M$};
    \node[rv] (U) at (2.2,0) {$U_i$};
    \node[rv] (W) at (2.2,-1) {$W_i$};

    \node[rv] (T) at (4.7,1) {$\overline{W}_i$};
    \node[rv] (Y) at (4.7,-1) {$W^\emptyset_i$};

    \draw[edge] (S) -- node[lab,above,xshift=-2mm] {$\learning(\cdot)$} (M);
    \draw[edge] (S) -- (U);
    \draw[edge] (S) -- (W);

    \draw[edge] (U) -- (W);
    \draw[edge] (U) -- (T);

    \draw[edge] (M) -- node[lab,above] {$\unlearning(\cdot,U_i)$} (T);
    \draw[edge] (W) -- node[lab,above] {$\unlearning(\cdot,\emptyset)$} (Y);

    \end{tikzpicture}
    \caption{
    Markov graph for a single deletion request \(U_i\).
    }
    \label{fig:markov-graph}
\end{wrapfigure}

\paragraph{Setup.}
Draw
\(
    S \sim \mathcal{P}_S
\)
and allow $\learning$ to be randomized.
The requests 
\(
    \mathcal{U} = \bc{
        U_1, \dots, U_K
    }
\)
are deterministic functions of $S$, with $U_i \subseteq S$
\footnote{
    They can also be randomized; the distribution of $U_i$ must be independent of the algorithmic randomness.
}.
For each $i \in [K]$, define
\(
    M \coloneqq \learning(S),
\)
\(
    W_i \coloneqq \learning(S \setminus U_i),
\)
\(
    \overline{W}_i \coloneqq \unlearning(\learning(S), U_i),
\)
\(
    W^\emptyset_i \coloneqq \unlearning(\learning(S \setminus U_i), \emptyset).
\)
Thus, $M$ is the pretrained model, $W_i$ is the retraining target for the request $U_i$,
$\overline{W}_i$ is the unlearning output,
while $W^\emptyset_i$ is the empty-request output of $W_i$.
For a tuple
\(
    A^K \coloneqq \br{
        A_1, \dots, A_K
    }
\)
and a permutation $\pi: [K] \mapsto [K]$,
write
\(
    A_{\pi(\leq i)}
    \coloneqq \br{
        A_{\pi(1)}, \dots, A_{\pi(i)}
    }
\)
and
\(
    A_{\pi(< i)}
    \coloneqq \br{
        A_{\pi(1)}, \dots, A_{\pi(i-1)}
    }.
\)
When $i=1$, $A_{\pi(<1)}$ is the empty tuple.
\Cref{fig:markov-graph} records the conditional
independences used below: $M$ is independent of $(U_i, W_i)$ given $S$,
$\overline{W}_i$ depends on the rest only through $(M, U_i)$,
and $W^\emptyset_i$ depends only on $W_i$.

\vspace{.5em}
\noindent The following proposition serves as a basis for deriving the concrete memorization lower
bound.

\begin{proposition}\label{pp:information-lb}
Let \(S\sim \mathcal{P}_S\), and let
\(U_1,\ldots,U_K\) be requests generated as above. For each
\(i\in\bs{K}\), define \(M, W_i,\overline{W}_i\) as above. Then, for every permutation \(\pi:[K]\to[K]\),
\[
I(M;S)
\ge
\sum_{i=1}^K
I\!\left(
W_{\pi(i)};\overline{W}_{\pi(i)}
\,\middle|\,
W_{\pi(<i)},U_{\pi(\le i)}
\right).
\]
\end{proposition}

\subsection{Memorization Lower Bound}\label{sec:lb-discrete}

We consider a discrete, finite model space \(\cW\).  We measure how
well the unlearned model $\overline{W}$ recovers the retraining target
$W$ by the probability that the two models differ.  Two sources of
error contribute to this: (i) the empty-request unlearning may move
the target,  $(d_0, \Gamma)$-utility bounds this by $\eta \coloneqq
\Gamma(0)$ (with $d_0(w, w') \coloneqq \dsI\bs{w \neq w'}$; see
\Cref{def:gamma-utility}); (ii) and the unlearned model may differ
from the empty-request output $W^\emptyset = \unlearning(W,
\emptyset)$, which is controlled by R\'enyi unlearning.  

\begin{definition}[Recovery error]
Recall that $W = \learning(S\setminus U)$ and $\overline W = \unlearning(M,U)$.
The \textit{unlearning recovery error} is
\[
    \beta_{\mathrm{un}} \coloneqq \sup_{\substack{s,u:\, u\subseteq s\\ |u|\le m}}
    \mathbb P\big(\overline W \ne W \mid S=s,\, U=u\big),
\]
where the probability is over the randomness of $\learning$ and
$\unlearning$, with independent runs of $\learning$ for fixed $s,u$.
\end{definition}

\noindent The following lemma bounds the unlearning recovery error.

\begin{lemma}
\label{lem:t-recovery-bound}
Assume that \((\learning,\unlearning)\) satisfies
\(\br{\alpha,\varepsilon}\)-R\'enyi unlearning for all deletion
requests of size at most \(m\) and \(\unlearning\) has
\((d_0,\Gamma)\)-utility on empty requests with \(\eta\coloneqq\Gamma(0)\).
Define \(\gamma\coloneqq\frac{\alpha-1}{\alpha}\). Then,
\[
    \beta_{\mathrm{un}} \leq \bar{\beta}\coloneqq \min \bc{1, e^{\gamma \varepsilon} 
    \br{\eta + \Delta_{\mathcal A}}^\gamma}, \quad 
    \Delta_{\learning}\coloneqq \sup_{\substack{s,u:\,u\subseteq
    s\\ |u|\le m}} \prob{ W_{s,u}^{(1)}\ne W_{s,u}^{(2)} },
\]  
where \(W_{s,u}^{(1)}\) and \(W_{s,u}^{(2)}\) are independent runs of \(\learning(s\setminus u)\).
\end{lemma}

\noindent Now, we state our main lower bound in the discrete setting.

\begin{theorem}
\label{thm:information-lb-discrete}
    Let $S \sim \cP_S$, and $U_1, \dots, U_K$ be a set of unlearning requests generated as above, each of size at most $m$. Assume $\mathcal{W}$ is finite and that $(\learning$, $\unlearning)$ satisfy the same assumptions as in \Cref{lem:t-recovery-bound}, with $\bar{\beta}$ as defined there. Set $\tilde{\beta} \coloneqq \min\bc{\bar{\beta},\, 1 - \frac{1}{|\mathcal{W}|}}$. %
    Then, for any permutation $\pi : [K] \to [K]$,
    \[
        I(M;S) \geq \sum_{i=1}^{K} H\br{W_{\pi(i)} \mid W_{\pi(<i)}, U_{\pi(\le i)}}
        - K\br{h(\tilde{\beta}) + \tilde{\beta}\log\br{|\mathcal{W}| - 1}}, %
    \]
    where $h$ denotes the binary entropy. %
\end{theorem}

\begin{proof}[Proof sketch]
    Apply~\Cref{pp:information-lb}. It remains to lower bound each term
    $I(W_{\pi(i)};\overline W_{\pi(i)}\mid \cH_i)$, where
    $\cH_i \coloneqq (W_{\pi(<i)},U_{\pi(\le i)})$. Use the definition of mutual information and the fact that conditioning reduces entropy, to get
    $I(W_{\pi(i)};\overline W_{\pi(i)}\mid \cH_i)
    \ge H(W_{\pi(i)}\mid \cH_i) - H(W_{\pi(i)}\mid \overline W_{\pi(i)})$.
    \Cref{lem:t-recovery-bound} gives $\mathbb P(\overline W_{\pi(i)}\ne W_{\pi(i)})\le\bar\beta$, and
    Fano's inequality gives
    $H(W_{\pi(i)}\mid \overline W_{\pi(i)})\le h(\tilde\beta)+\tilde\beta\log(|\mathcal W|-1)$.
    Plugging this back into the mutual information bound and summing over $i$ proves the theorem.
\end{proof}

The bound is informative when the entropy terms are large and the penalty
is small. Each entropy term measures the new information a retraining
target reveals beyond the previously accounted requests and targets, in
the spirit of an inclusion-exclusion principle; since the unlearner is not
run sequentially, \(\pi\) is only an accounting device that decomposes this
information across requests handled in parallel, and, as the bound holds
for every ordering, we choose the one giving the strongest guarantee.
Entropy terms grow with the size of the model space and the algorithmic
randomness of \(\learning\), since a fixed input then yields a more
uncertain output. The penalty term captures the loss from approximate
unlearning: it grows when the unlearning requirement is weaker (e.g.,\ a
larger \(\varepsilon\)) or the empty-request map is noisier, allowing a
looser approximation of the retraining target; it also scales with the
model-space size and \(\learning\)'s randomness, normalizing it against the
entropy terms.

\begin{remark}
    Since the Rényi unlearning guarantee holds for every dataset and every
    admissible deletion request, the memorization bound applies to any family of
    requests, up to the family of all nonempty subsets of $S$ of size at most $m$.
    Most of these requests, however, may not reveal new information about the retraining
    targets, while each still contributes a penalty term. In the examples below, we
    therefore choose a small family of requests, each revealing new information, so
    that $K$ is much smaller than the number of admissible requests.
\end{remark}

\subsection{Instantiations of~\Cref{thm:information-lb-discrete}}\label{sec:examples-disc}

In this subsection, we provide various instantiations of the general
bound in~\Cref{thm:information-lb-discrete}. We assume the unlearning
algorithm satisfies \((\alpha,\varepsilon)\)-R\'enyi unlearning for
every deletion set of size at most \(m\). Let
\(\gamma\coloneqq(\alpha-1)/\alpha\), and assume the discrete
empty-request utility condition
\(\prob{\bar{\learning}(w,\emptyset)\neq w\mid W=w}\le \eta\) for each
retraining target \(w\). Before stating the examples, we first
describe some straightforward improvement over the general theorem by
applying a conditional version of the Fano's inequality. For \(q\ge2\),
define %
\[
    \beta_q \coloneqq \min\Bigl\{e^{\gamma\varepsilon}\eta^\gamma,\, 1-\tfrac{1}{q}\Bigr\},
    \qquad
    h(p)
    \coloneqq
    - p \log p - (1-p) \log (1-p).
\]

For the deterministic learners below, the recovery error is at most
\(e^{\gamma\varepsilon}\eta^\gamma\) for every dataset, so this bound
also holds conditional on \(\mathcal H_i=(W_{<i},U_{\le i})\). If
\(W_i\) has at most \(q_i\ge2\) possible values given this history,
any estimate outside that set can be mapped into it without increasing
the error.  Applying Fano's inequality conditional on \(\mathcal H_i\)
therefore gives the penalty \(h(\beta_{q_i})+\beta_{q_i}\log(q_i-1)\),
even when the full model space is infinite. If only one target is
possible, the conditional entropy is zero.

The propositions below state necessary conditions for the existence of
a forget-only unlearning algorithm satisfying the assumptions above.
Thus, if the usual learner output contains less information than the
displayed lower bound, the conclusion is that the usual output alone
cannot support such forget-only unlearning; a successful
implementation would have to store an additional state. All proofs are
deferred to~\Cref{app:example-proofs}.

\noindent\textbf{Canonical threshold ERM learner.}
\label{ex:canonical-thresholds}
Let \(\mathcal Z=\mathcal X\times\{0,1\}\), where \(\mathcal X=[N]\),
and consider \(\cH=\bc{h_a : h_a(x)=\mathds{1}\{x\ge a\},\; a\in[N+1]}.\) The canonical ERM tie-breaking rule returns the leftmost
point labeled \(1\), or \(N+1\) if there is no positive example.

\begin{proposition}[Canonical threshold ERM]
\label{pp:canonical-thresholds}
    Let \(n\ge 2m\), let \(q\ge2\), and set \(N=nq\). Then there
    exists a distribution \(\cP_S\) over realizable threshold samples
    of size \(n\), and there are deletion requests \(U_1,\ldots,U_m\),
    each of size exactly \(m\), such that any forget-only unlearning
    algorithm satisfying the assumptions above must obey
    \[
        I(M;S) 
        \geq
        m \bs{
            \log \frac{N}{n}
            -
            \beta_q
            \log \br{
                \frac{N}{n} - 1
            }
            -
            h(\beta_q)
        }.
    \]

    For the bare canonical ERM output on thresholds, one has
    \(I(M;S)_{\mathrm{bare}}=\log\br{N/n}\). %
\end{proposition}

\noindent\textbf{Coordinate PCA.} We consider uncentred PCA restricted
to output coordinate directions. Given a dataset \(S\), the learner
returns the basis vector \(e_j\) maximising \(\sum_{x\in S}x_j^2\).
Its output is therefore just a single coordinate index and requires
only \(O(\log d)\) bits to store. Below, we prove that forget-only
unlearning can require retaining substantially more information.

\begin{proposition}[Coordinate PCA]
\label{prop:coordinate-pca}
    Consider top-coordinate PCA in dimension \(d\). Let \(n\) be the
    dataset size and let \(m\) be the deletion budget. Assume \(d=nq\)
    for some integer \(q\ge2\), and assume \(2m\le n\). Then there
    exists a distribution \(\cP_S\) over datasets of size \(n\), and
    deletion requests \(U_1,\ldots,U_m\), each of size exactly \(m\),
    such that any forget-only unlearning algorithm satisfying the
    assumptions above must obey\[ I(M;S)
        \ge
        m\bs{
            \log\br{\frac dn}
            -
            \beta_q \log\br{\frac dn-1}
            -
            h(\beta_q)
        }.
    \]

    For the bare top-coordinate PCA output, one has
    \(I(M;S)_{\mathrm{bare}}=\log \br{d/n} \).
\end{proposition}

\noindent\textbf{Ridge regression.}
Consider ordinary ridge regression with a fixed penalty \(\lambda>0\):
\[
    \learning_{\mathrm{ridge}}(D)
    =
    \argmin_{w\in\reals^d}
    \bc{
        \frac12\sum_{(x,y)\in D}
        \br{\ip{w}{x}-y}^2
        +\frac{\lambda}{2}\abs{w}_2^2
    }.
\]
Consider the input domain to be one-hot vector. Each input represents
one of \(d\) categories. An example from category \(j\) is encoded by
the one-hot vector \(e_j\), whose \(j\)-th coordinate is \(1\) and all
other coordinates are \(0\).  The labels take values in
\(\bc{-1,+1}\).  Since \(\ip{w}{e_j}=w_j\), ridge regression
fits one coefficient for each category.

\begin{proposition}[Ridge regression]
\label{prop:ridge-regression}
    Fix an integer \(b\ge1\), set \(d=2b\) and \(n=8b\). There exists a distribution
    \(\cP_S\) over datasets~(as described above) of size \(n\), and
    singleton deletion requests \(U_1,\ldots,U_b\), such that any
    forget-only unlearning algorithm must obey
    \[
        I(M;S)
        \ge
        \frac{n}{8}
        \left[
            \log3-h(\beta_3)-\beta_3\log2
        \right].
    \]
    On the distribution \(\cP_S\), the bare ridge-regression output is
    identically zero, so \(I(M;S)_{\mathrm{bare}}=0\).
\end{proposition}

Thus, whenever \(\beta_3<2/3\), the ordinary fitted
coefficients alone cannot support the required guarantee, even
for single-example deletions.

\noindent\textbf{Row-wise factorized affine matrix completion.} Fix
\(r\) and view the first \(r^2\) coordinates as an \(r\times r\)
matrix \(X\) with rows \(x_i\in\reals^r\). A data point
\(\br{i,a,y}\), with \(i\in[r]\), \(a\in\reals^r\), and
\(y\in\reals\), imposes the row-wise affine constraint \(
\ip{a}{x_i}=y.\) Thus \(\br{i,\bf1_r,r}\) imposes the row-sum
constraint, and \(\br{i,e_\ell,1}\) imposes \(\br{x_i}_\ell=1\). The
learner uses a row-wise factorization \(x_i=u_i v_i\), with
\(u_i\in\reals\) and \(v_i\in\reals^r\), and minimizes
\(
    \frac12\sum_{\br{i,a,y}\in S}
    \br{\ip{a}{u_i v_i}-y}^2
    +
    \frac{\lambda}{2}
    \sum_{i=1}^r
    \br{u_i^2+\norm{v_i}_2^2}.
\)
The algorithm \(\learning_{\mathrm{fac}}\) returns the fitted matrix
in the limit \(\lambda\to0\). On the consistent datasets below, this
limit is the interpolating matrix minimizing \(\sum_{i=1}^r
\|x_i\|_2,\) because
\(
    \inf_{u_i v_i=x_i}
    \frac12(u_i^2+\|v_i\|_2^2)
    =
    \|x_i\|_2.
\)

This is best viewed as a row-wise affine matrix-completion problem
with a factorized norm bias. The construction below shows that
forget-only unlearning may force the trained object to store a hidden
permutation even though the ordinary fitted matrix on the full sample
is deterministic.

\begin{proposition}[Row-wise factorized affine matrix completion]
\label{prop:factorized-matrix-completion}
    Let \(d>4\) be the ambient number of scalar matrix entries, let
    \(m\ge1\) be the deletion budget, and let \(n>3m\) be the dataset size.
    Define \(
        r_\star
        \coloneqq
        \min\bc{ \left\lfloor\sqrt d\right\rfloor,
            \left\lfloor\frac{n}{m+1}\right\rfloor}.  \) Then there
            exists a distribution \(\cP_S\) over datasets of size \(n\),
            and deletion requests \(U_1,\ldots,U_{r_\star}\), each of
            size exactly \(m\), such that any forget-only unlearning
            algorithm satisfying the assumptions above must obey
    \[
        I(M;S)
        \ge
        \log(r_\star!)
        -
        \sum_{s=2}^{r_\star}
        \bs{
            h(\beta_s) + \beta_s \log(s-1)}.
    \]
    On \(\cP_S\), the basic fitted matrix is deterministic and thus
    \(I(M;S)_{\mathrm{bare}}=0\) when no unlearning is required.
\end{proposition}

\begin{remark}
    We would like to highlight a consequence of our lower bounds, which can hold for a variety of algorithms. For instance, in the case of the canonical threshold ERM learner, the algorithm is simply outputting a real number - the threshold - from a finite set of possible thresholds. Unlearning, on the other hand, imposes the requirement to retain significantly more information than just this single number. However, there is no way to encode all of this additional information into that one number while remaining within the original model space. This implies that we may need to move to a different output space in order to output this additional information, and possibly design a different algorithm for the unlearning problem.
\end{remark}

\section{Related Work and Discussion}\label{sec:discussion}

Machine unlearning asks for a cheaper alternative to retraining whose
output behaves like a model trained from scratch without the deleted
data~\citep{cao2015towards,guo2019certified}. We study the stricter
forget-only setting: at deletion time, the unlearner receives only the
trained model \(M=A(S)\) and the forget set \(U\), not the retained
set \(S\setminus U\), training logs, checkpoints, gradients, or
auxiliary statistics. Several empirical methods are forget-only, or
close to forget-only, in this operational sense. Some zero-shot
unlearning methods update the trained model using only the forget
examples~\citep{chundawat2023zero,foster2024zeroshot}. Related
single-point analyses of SGD show that approximate unlearning
guarantee depends on the stability of the training
dynamics~\citep{thudi2021unrolling}. Recent work has also developed nearly forget-only unlearning methods
for language models. However, these works use notions of unlearning that differ
from ours; we discuss them, together with other empirical evaluations,
in~\Cref{sec:app-related}. The empirical literature shows that
forget-only and retain-free unlearning can be useful, but our theory
explains why this does not imply universal feasibility. Existing
successes typically rely on at least one of a few common forms of structure:
the request is benign, the model is locally stable, retain information
is approximated by a proxy, or the unlearning definition is weaker
than retraining-style indistinguishability.
 This is consistent with our results: the lower bounds do
not rule out forget-only unlearning, but they require the retraining
target to be recoverable from the information available in \((M,U)\).

A related line removes access to the original retain data but
reconstructs a proxy. Source-free and retain-free methods synthesize
adversarial, energy-guided, or generator-based proxy data to preserve
utility while unlearning~\citep{mu2024rewind,chen2025zspag,wang2025dsda,
lee2025ruago}. Other works obtain unlearning guarantees using a
surrogate distribution or model-side curvature information~\citep{
basaran2025certified,ahmed2025sourcefree}. Several theoretical works
allow the unlearner to use training summaries, gradients, or other
auxiliary local information in order to obtain certified unlearning
algorithms~\citep{ginart2019making,neel2021descent,
sekhari2021RememberWhatYou,heinzler2026less}. These methods are close
in motivation to ours, but the missing retain information is replaced
by a proxy, a surrogate, a local approximation, or a state deliberately
preserved during training. Our memorization lower bound intuitively
captures the same intuition: the trained model must memorize enough
recoverable information to reconstruct the retraining target.

The closest theoretical comparison is the work of
\citet{cherapanamjeri2025space}, who lower bound the memory complexity
of learning--unlearning algorithms in terms of the \emph{eluder dimension}.
Their result is complementary to ours: it studies realizability testing,
with extensions to ERM using exact minimizers as the retraining reference,
and is governed by the hypothesis class, whereas
our bounds are learner-dependent and unlearning request-dependent, and
apply directly to strict forget-only unlearning without auxiliary
information for any \(\br{\alpha,\varepsilon}\). In particular, their
results do not distinguish between learning algorithms with the same
hypothesis class, and do not directly cover tasks such as regression,
PCA, and factorization.

Our lower bounds quantify how much recoverable information must be
memorized in the trained model for forget-only unlearning to
succeed, but they do not identify the form that this information should
take. In particular, they do not specify which features or statistics should be memorized, nor how such information should be stored. This leaves open a natural question: whether the memorization already observed in overparameterized models can
be accessed in a way that supports certified forget-only
unlearning.  Finally, instantiating these lower bounds for contemporary architectures
and large-scale learning pipelines remains an important direction for
understanding when empirical failures of forget-only unlearning are
algorithmic and when they are information-theoretic.

\clearpage

\subsection*{Acknowledgments}
LR and AS acknowledge support from VILLUM FONDEN via the Young
Investigator program (72069). AS and VS also acknowledge the Novo
Nordisk Foundation for support via the Startup grant (NNF24OC0087820).

\bibliography{bibliography}
\bibliographystyle{plainnat}

\appendix
\clearpage
\section{Additional Related Works}
\label{sec:app-related}

Recent work on language-model unlearning has produced several
forget-only, or nearly forget-only, procedures.  Early work on
knowledge unlearning uses gradient-ascent or unlikelihood objectives
on target sequences to reduce memorization
risk~\citep{jang2023knowledge}.  \citet{yao2023large} view LLM
unlearning as a post-hoc alignment operation driven by negative
examples, while \citet{eldan2023whos} fine-tune a large model so that
it stops recalling Harry Potter content while largely preserving
standard benchmark performance.  More recent methods refine this
optimization perspective: \citet{zhang2024npo} introduce negative
preference optimization to avoid the catastrophic collapse often
caused by plain gradient ascent, and \citet{wang2025flat} propose
FLAT, a forget-data-only objective that uses neither retain data nor a
reference model.  These works optimize empirical forgetting criteria
rather than the stronger indistinguishability requirement studied
here. They therefore leave open whether weaker, task-level notions of
forgetting can be achieved with substantially less memorization.

Another line of work shows that the difficulty of unlearning depends
on the deletion distribution and on the evaluation criterion.
\citet{allouah2025utility} distinguish in-distribution from
out-of-distribution deletion requests: simple certified procedures can
achieve tight tradeoffs for in-distribution deletions, whereas
out-of-distribution deletions can be harder than retraining in some
regimes.  For language models, benchmarks such as TOFU and MUSE show
that forgetting synthetic biographies, books, or news articles is
already difficult under empirical
metrics~\citep{maini2024tofu,shi2025muse}.  \citet{thaker2025weak}
argue further that the usual split into independent forget and retain
queries can overstate progress, since realistic queries may couple
forgotten and retained information.  Representation-based evaluations
lead to a similar conclusion: methods that appear successful under
logit- or accuracy-based metrics may leave internal representations
close to those of the original model, or may destroy the
representation altogether~\citep{kim2026truly}.  These empirical
observations are consistent with our lower bound: unlearning becomes
harder when the retraining target is more sensitive to retained data,
and when the criterion used to compare against retraining is more
stringent. A related distinction between deletion-specific and
distribution-free sensitivity is made in the certified setting by
\citet{heinzler2026less}.

\section{Supporting Lemmas}
\label{sec:supporting-lemmas}

\begin{lemma}[Event transfer]\label{lem:renyi-event-transfer}
Let \(\alpha>1\), let \(\gamma=\frac{\alpha-1}{\alpha},\) and suppose
\(D_\alpha(P\|Q)\le \varepsilon\). Then, for every measurable event
\(E\),
\(
    P(E)
    \le
    e^{\gamma\varepsilon} Q(E)^\gamma .
\)
\end{lemma}

\begin{proof}[Proof of~\Cref{lem:renyi-event-transfer}]
    Let \(\mu\) be a common dominating measure for \(P\) and \(Q\), meaning that \(P\) and \(Q\) are both absolutely continuous with respect to \(\mu\): if \(\mu(A)=0\), then \(P(A)=Q(A)=0\). Such a common dominating measure always exists; for example, one may take \(\mu = 1/2(P+Q)\). Define
    \[
    p \coloneqq \frac{dP}{d\mu},
    \qquad
    q \coloneqq \frac{dQ}{d\mu}.
    \]
    The Rényi divergence can be written as,
    \begin{equation}
        \renyi{\dvg{P}{Q}} = \frac{1}{\alpha-1} \log \int p(x)^\alpha q(x)^{1-\alpha} d\mu(x).
    \end{equation}
    Define the Radon-Nikodym derivative of $P$ w.r.t. $Q$ as,
    \begin{equation}
        L(x) \coloneqq \frac{dP}{dQ} (x) = \frac{p(x)}{q(x)}.
    \end{equation}
    Then we have,
    \begin{equation}
        \int L(x)^\alpha dQ(x) 
        = \int \left( \frac{p(x)}{q(x)} \right)^\alpha q(x) d\mu(x)
        = \int p(x)^\alpha q(x)^{1 - \alpha} d \mu(x),
    \end{equation}
    which implies that,
    \begin{gather}
    \begin{aligned}
        \renyi{\dvg{P}{Q}} 
        = \frac{1}{\alpha-1} & \log \int L(x)^\alpha dQ(x)
        = \frac{1}{\alpha - 1} \log \expp_Q (L^\alpha)
        \leq \varepsilon \\
        & \implies \expp_Q (L^\alpha) \leq e^{(\alpha-1)\varepsilon}.
    \end{aligned}
    \end{gather}
    For any event $A$, we have,
    \begin{equation}
        P(A) = \int_A dP(x) = \int_A L(x) \; dQ(x) = \expp_Q (L \; \mathds{1}_A).
    \end{equation}
    Using  H\"{o}lder's inequality with $m=\alpha$ and $n=\frac{\alpha}{\alpha-1}$ ($\frac{1}{m} + \frac{1}{n} = 1$), we finally get, 
    \begin{gather}
    \begin{aligned}
        P(A)
        &= \expp_Q (L \; \mathds{1}_A) 
        \leq (\expp_Q L^m)^{1/m} \; (\expp_Q \mathds{1}_A^n)^{1/n} \\
        &= (\expp_Q L^\alpha)^{1/\alpha} \; (\expp_Q \mathds{1}_A)^{(\alpha-1)/\alpha}
        \leq e^{(\alpha-1)/\alpha \; \varepsilon} \; Q(A)^{(\alpha-1)/\alpha}.
    \end{aligned}
    \end{gather}
    This completes the proof.
\end{proof}

\section{Deferred Proofs from Section~\ref{sec:impossibility}}
\label{sec:impossibility-proofs}

\begin{lemma}
\label{lem:renyi-from-concentration}
    Suppose \(\unlearning\) has
    \(\br{\phi,\Gamma}\)-utility on empty requests. 
    Consider two outputs of $\learning$: $W_1, W_2$;
    with laws $P_1, P_2$. 
    Fix two disjoint sets $E_1, E_2$ with
    \begin{equation}
    \label{eq:dist-condition}
        \br{
            \forall z_1 \in E_1
        }
        \br{
            \forall z_2 \in E_2
        }
        \quad
        \phi( z_1 - z_2 ) > 2r,
    \end{equation}
    such that
    \(
        \prob_{W_i \sim P_i} \br{
            W_i \in E_i
        }
        \geq 1 - \kappa
    \)
    for $i \in \bc{1,2}$.
    Define the sets 
    \[
        B_i \coloneqq \bc{
            w \in \cW: \;
            \br{
                \exists z \in E_i
            }\quad
            \phi(w - z) \leq r
        },
    \]
    and 
    $Q_i \coloneqq \law \br{\unlearning(W_i, \emptyset)}$ for $i \in \bc{1,2}$. Then, it holds that $B_1 \cap B_2 = \emptyset$, and,
    \[
        Q_i (B_i) \geq 
        1 - \kappa - \Gamma(r).
    \]
\end{lemma}
\begin{proof}
    First, note that $B_1$ and $B_2$ are disjoint. Suppose $y\in B_1\cap B_2$. Then for every $\epsilon'>0$ there are $w_i\in E_i$ with $\phi(y-w_i)\le r+\epsilon'$. Hence $\phi(w_1-w_2)\le2r+2\epsilon'$, which contradicts~\eqref{eq:dist-condition} once $\epsilon'$ is small enough.

    For $w \in E_i$, the ball
    \(
        \bc{
            y: \phi(y - w) \leq r
        }
    \)
    is contained in $B_i$. By~\Cref{def:gamma-utility}, it holds that
    \(
        Q_w (B_i) \geq 1-\Gamma(r)
    \)
    for $i \in \bc{1,2}$.
    Integrating over $w \sim P_i$ and keeping only $w \in E_i$, we have,
    \[
    \begin{aligned}
        Q_i(B_i) 
        &\geq
        \int_{E_i} Q_w(B_i) \; P_i(dw)
        \geq \br{
            1 - \Gamma(r)
        }
        \int_{E_i}
        P_i(dw) \\
        &=
        \br{
            1 - \Gamma (r)
        }
        P_i(E_i)
        \geq
        \br{
            1 - \Gamma(r)
        }
        \br{
            1 - \kappa
        }
        \geq
        1 - \kappa - \Gamma(r).
    \end{aligned}
    \]
\end{proof}

\begin{proof}[Proof of \Cref{thm:concentration-impossibility}]
    \label{app:proof-concentration-impossibility}
    For $i \in [N]$, write 
    $M_i \coloneqq \learning(S_i)$,
    $W_i \coloneqq \learning(S_i \setminus U)$,
    $P_i \coloneqq \law(W_i)$,
    $Q_i \coloneqq \law (\unlearning(W_i, \emptyset))$,
    and $T_i \coloneqq \law(\unlearning(M_i, U))$.
    Define
    \(
        B_i \coloneqq \bc{
            w \in \cW: \;
            \br{
                \exists z \in E_i
            }\quad
            \phi(w - z) \leq r
        },
    \)
    and $\gamma = \frac{\alpha-1}{\alpha}$.
    
    Recall that $u(r) \coloneqq \kappa + \Gamma(r)$, and fix $r \in R$, so $0 < r < \Delta/2$ and $0 < u(r) < 1$.
    For each pair $i \neq j$, the sets $E_i, E_j$ from the
    shared deletion example satisfy
    \[
        \br{
            \forall z_i \in E_i
        }
        \br{
            \forall z_j \in E_j
        }
        \quad
        \phi( z_i - z_j ) \geq \Delta > 2r,
    \]
    and
    \(
        \prob_{W_k \sim P_k} \br{W_k \in E_k} \geq 1 - \kappa
    \)
    for $k \in \bc{i, j}$.
    Applying~\Cref{lem:renyi-from-concentration} to the laws $P_i, P_j$ and the
    sets $E_i, E_j$, we get that $B_i \cap B_j = \emptyset$ and
    $Q_i(B_i) \geq 1 - u(r)$, implying $Q_i(B_i^c) \leq u(r)$.
    Since this holds for every pair $i \neq j$, the sets $B_1, \dots, B_N$ are
    pairwise disjoint.

    The unlearning assumption gives $d_\alpha (Q_i, T_i) \leq \varepsilon$ for all $i \in [N]$.
    As $T_i$ is the image of $M_i$ under the map $w \to \unlearning(w, U)$, by the data processing inequality, we have that, for $i \neq j$,
    \[
        d_\alpha(T_i, T_j) \leq d_\alpha(M_i, M_j) \leq \delta_\alpha.
    \]
    Specifically, we will use $d_\alpha(T_1, T_i) \leq \delta_\alpha$.

    As the sets $B_1, \dots, B_N$ are disjoint, it holds that,
    \(
        \sum_{i=1}^N T_1(B_i) \leq 1.
    \)

    We first analyze the exact collision case, which also covers the deterministic learners. The zero-collision condition forces $T_i = T_1$, which implies $d_\alpha(Q_i, T_1) \leq \varepsilon$ for all $i \in [N]$.
    Using~\Cref{lem:renyi-event-transfer} in two directions, together with the unlearning assumption, we have,
    \[
    \begin{aligned}
        1 - u(r) \leq Q_i(B_i) \leq e^{\gamma \varepsilon} T_1(B_i)^\gamma
        &\implies
        T_1(B_i) \geq e^{-\varepsilon} \br{1 - u(r)}^{1 / \gamma}, \\
        1 - T_1(B_i) = T_1(B_i^c) \leq e^{\gamma \varepsilon} Q_i(B_i^c)^\gamma \leq e^{\gamma \varepsilon} u(r)^\gamma
        &\implies
        T_1(B_i) \geq 1 - e^{\gamma \varepsilon} u(r)^\gamma
    \end{aligned}
    \]
    Summing each lower bound over $i$ and using
    \(
        \sum_{i=1}^N T_1(B_i) \leq 1,
    \)
    we get the final two terms depending on $r$.

    In the approximate collision case, we use~\Cref{lem:renyi-event-transfer} twice to get,
    \[
    \begin{aligned}
        &1 - u(r) 
        \leq Q_i(B_i) 
        \leq e^{\gamma \varepsilon} T_i(B_i)^\gamma
        \leq e^{\gamma \varepsilon + \gamma^2 \delta_\alpha} T_1(B_i)^{\gamma^2}
        \implies
        T_1 (B_i) \geq \br{
            1 - u(r)}^{1 / \gamma^2}
         e^{-\varepsilon/\gamma - \delta_\alpha}, \\
        &T_1(B_i^c) 
        \leq e^{\gamma \delta_\alpha} T_i(B_i^c)^\gamma
        \leq e^{\gamma^2 \varepsilon + \gamma \delta_\alpha}
        Q_i(B_i^c)^{\gamma^2}
        \leq e^{\gamma^2 \varepsilon + \gamma \delta_\alpha} u(r)^{\gamma^2}
        \implies
        T_1(B_i) \geq 1 - e^{\gamma^2 \varepsilon + \gamma \delta_\alpha} u(r)^{\gamma^2},
    \end{aligned}
    \]
    Here we used that $x \mapsto x^\gamma$ is increasing. Summing the first lower bound over the disjoint sets $B_i$ gives
    \[
        1 \geq \sum_{i=1}^N T_1(B_i)
        \geq N \br{1-u(r)}^{1/\gamma^2}
        e^{-\varepsilon/\gamma-\delta_\alpha}.
    \]
    Taking logarithms and rearranging,
    \[
    \begin{aligned}
        \frac{\varepsilon}{\gamma}
        &\geq \log N - \delta_\alpha
        + \frac{1}{\gamma^2}\log(1-u(r)), \\
        \varepsilon
        &\geq \gamma\br{\log N-\delta_\alpha}
        + \frac{1}{\gamma}\log(1-u(r)).
    \end{aligned}
    \]
    Similarly, summing the second lower bound gives
    \[
        1 \geq \sum_{i=1}^N T_1(B_i)
        \geq N\br{1-e^{\gamma^2\varepsilon+\gamma\delta_\alpha}u(r)^{\gamma^2}},
    \]
    hence
    \[
        e^{\gamma^2\varepsilon+\gamma\delta_\alpha}u(r)^{\gamma^2}
        \geq \frac{N-1}{N}.
    \]
    Taking logarithms and dividing by $\gamma^2$ yields
    \[
        \varepsilon \geq
        \frac{1}{\gamma^2}\log\frac{N-1}{N}
        -\frac{\delta_\alpha}{\gamma}-\log u(r).
    \]

    In both cases, taking the supremum over $\bc{r \in (0, \Delta/2): 0 < u(r) < 1}$ gives the final bounds.

\end{proof}

\subsection{Teaching Dimension}
\label{subsec:teaching-dimension}

\begin{definition}[Teaching dimension~\citep{goldman1995complexity}]
\label{def:teaching-dim}
    Let $\cX$ be an instance space, and let $\cH \subseteq \bc{0,1}^{\cX}$ be a hypothesis class over $\cX$.
    For a hypothesis $h \in \cH$, a set $T \subseteq \cX \times \bc{0,1}$ is called a \textit{teaching set} for $h$ if $h$ is consistent with $T$ and no other hypothesis $h' \in \cH \setminus \{h\}$ is consistent with $T$.
    Let $\mathcal{T}(h)$ denote the collection of all teaching sets for $h$.
    The \textit{teaching dimension} of $\cH$ is defined as
    \[
        \mathrm{TD}(\cH)
        \coloneqq
        \max_{h \in \cH}
        \min_{T \in \mathcal{T}(h)}
        |T|.
    \]
\end{definition}

\begin{proof}[Proof of~\Cref{pp:teaching-dimension}]
    Let $h_0 \in \cH$ have a teaching set $T_0$. A teaching set size is upper-bounded by the teaching dimension, hence $|T_0| \leq \tau$. Similarly, let $R_i$ be teaching sets for $h_i$, with $|R_i| \leq \tau$ as well.
    Construct the common unlearning request as $\tau+1$ copies of $T_0$: $U = \uplus_{i=1}^{\tau+1} T_0$, so
    $|U| \leq \tau (\tau + 1)$. Set $S_i = R_i \cup U$, so $|S_i| \leq \tau (\tau + 2)$.

    Consider first the hypothesis $h_0$. On $S_i$, it errs only on the retain part $R_i$. $U$ consists of copies of $T_0$, which is perfectly fitted by $h_0$ by construction. As $|R_i| \leq \tau$, $h_0$ can make at most $\tau$ mistakes on it. On the other hand, $h_i \neq h_0$ perfectly fits the retain part $R_i$. However, on each of the $\tau+1$ copies of $T_0$ in $U$, it must make at least one mistake, as $h_0$ is the only hypothesis that perfectly fits $U$. Therefore, $h_i$ makes at least $\tau + 1$ mistakes on $S_i$. Combining these two results, we obtain that $\learning(S_i) = h_0$ for any $i \in [N]$. The retrained targets are $\learning(S_i \setminus U) = \learning(R_i) = h_i$, again by the teaching set property.

    Therefore, the above implies that $N$ different datasets exactly collide under any ERM $\learning$, but are strictly separated under retraining.
    
\end{proof}

\subsection{Proofs of the Examples}

\begin{proof}[Proof of \Cref{ex:thresholds}]
\label{app:proof-thresholds}
For a realizable sample \(S\), write
\(n(S)\coloneqq \max\{x:(x,-1)\in S\}\) and
\(p(S)\coloneqq \min\{x:(x,+1)\in S\}\). Thus
\(\learning_c^{\mathrm{th}}(S)=\br{n(S)+p(S)}/2\).

We first prove the lower bound. Let
\(
    U = \cup_{i=1}^m \bc{
        (-c+1, +1)
    }
\)
be the deletion multiset.
Define the following multisets:
\[
S_1
=
U \cup 
\bc{
    (-c, -1), (-c+1, +1)
}
\qquad
S_2
=
U \cup
\bc{
    (-c, -1), (c, +1)
}.
\]
Pad each dataset with \(n-m-2\) additional copies of \((-c,-1)\).
Both full samples have largest negative example \(-c\) and smallest
positive example \(-c+1\), so
\(\learning_c^{\mathrm{th}}(S_1)=\learning_c^{\mathrm{th}}(S_2)=-c+\frac12\). 
After deleting \(U\),
however, \(\learning_c^{\mathrm{th}}(S_1\setminus U)=-c+\frac12\) and
\(\learning_c^{\mathrm{th}}(S_2\setminus U)=0\).
Hence, \((S_1,S_2,U)\) is a $(2, m, c-\frac12, |\cdot|;0,0)$-shared-deletion example.

It remains to show that no larger separation is possible. Consider any
shared-deletion example \((S_1,S_2,U)\) with \(|U|\le m\). Put
\(a_i=n(S_i)\), \(b_i=p(S_i)\), \(a_i'=n(S_i\setminus U)\), and
\(b_i'=p(S_i\setminus U)\) for \(i\in\{1,2\}\). Since
\(\learning_c^{\mathrm{th}}(S_1)=\learning_c^{\mathrm{th}}(S_2)\), we have \(a_1+b_1=a_2+b_2\).
Assume without loss of generality that
\(\learning_c^{\mathrm{th}}(S_2\setminus U)\ge \learning_c^{\mathrm{th}}(S_1\setminus U)\).
Then
\[
\learning_c^{\mathrm{th}}(S_2\setminus U)-\learning_c^{\mathrm{th}}(S_1\setminus U)
\le
\frac{c-b_2}{2}
+
\frac{a_1+c}{2}.
\]
If \(b_2'=b_2\), then the first term is unnecessary and the bound is at
most \((a_1+c)/2\le c-\frac12\), since \(a_1\le c-1\). Similarly, if
\(a_1'=a_1\), then the bound is at most
\((c-b_2)/2\le c-\frac12\), since \(b_2\ge -c+1\).

It remains to consider the case \(b_2'>b_2\) and \(a_1'<a_1\). Then
\(U\) must contain \((b_2,+1)\), so \((b_2,+1)\in S_1\) and hence
\(b_2\ge b_1\). Also \(U\) must contain \((a_1,-1)\), so
\((a_1,-1)\in S_2\) and hence \(a_1\le a_2\). Together with
\(a_1+b_1=a_2+b_2\), these inequalities imply \(a_1=a_2\) and
\(b_1=b_2\).
Thus, writing \(a=a_1=a_2\) and \(b=b_1=b_2\),
\[
\learning_c^{\mathrm{th}}(S_2\setminus U)-\learning_c^{\mathrm{th}}(S_1\setminus U)
\le
\frac{a+c}{2}+\frac{c-b}{2}
=
c-\frac{b-a}{2}
\le
c-\frac12,
\]
because \(a,b\in\mathcal X_c\) and \(a<b\). Thus the largest admissible
separation in a \((2,m,\Delta,\abs{\cdot};0,0)\)-s.d.e. is
\(\Delta=c-\frac12\).

\end{proof}

\begin{proof}[Proof of \Cref{ex:median}]
For each \(i\in\{1,2\}\), write the multiset \(S_i\) in sorted order as
\(x_1^{(i)}\le \cdots \le x_n^{(i)}\).

We first consider the even case \(n=2r\), with \(r\in\mathbb N\). Since
\((S_1,S_2,U)\) is a shared-deletion example, the original medians coincide;
write \(M=\learning_c^{\mathrm{med}}(S_i)
=(x_r^{(i)}+x_{r+1}^{(i)})/2\), so
\(x_r^{(i)}\le M\le x_{r+1}^{(i)}\), for \(i\in\{1,2\}\).

\begin{enumerate}
    \item Case \(m=1\). Write \(U=\bc{u}\). Suppose first that \(u<M\).
    Then deletion removes a point from the left half of each dataset, so
    \(\learning_c^{\mathrm{med}}\br{S_i\setminus U}=x_{r+1}^{(i)}\). Since
    \(x_{r+1}^{(i)}=2M-x_r^{(i)}\), \(x_r^{(i)}\ge -c\), and
    \(x_{r+1}^{(i)}\le c\), we have
    \(\learning_c^{\mathrm{med}}\br{S_i\setminus U}\in
    \bs{M,\min\bc{c,2M+c}}\), an interval of length at most \(c\). The case
    \(u>M\) is analogous and gives
    \(\learning_c^{\mathrm{med}}\br{S_i\setminus U}\in
    \bs{\max\bc{-c,2M-c},M}\), again an interval of length at most \(c\).
    Finally, if \(u=M\), then necessarily \(x_r^{(i)}=x_{r+1}^{(i)}=M\), so
    \(\learning_c^{\mathrm{med}}\br{S_i\setminus U}=M\). Hence, in all cases,
    \(\abs{\learning_c^{\mathrm{med}}(S_1\setminus U)
    -\learning_c^{\mathrm{med}}(S_2\setminus U)}\le c\).

    This bound is tight. Take \(S_1=\bc{-c,0,\ldots,0}\), with \(2r-1\)
    copies of \(0\), and \(S_2=\bc{-c,\ldots,-c,c,\ldots,c}\), with \(r\)
    copies of \(-c\) and \(r\) copies of \(c\). Both medians are \(0\).
    Removing the common multiset \(U=\bc{-c}\) gives post-deletion medians
    \(0\) and \(c\), so the separation is \(c\). Therefore the largest
    admissible separation in a \((2,1,\Delta,\abs{\cdot};0,0)\)-s.d.e.
    is \(\Delta=c\).

    \item Case \(m=2\). Write \(U=\bc{u,v}\), with \(u\le v\). Consider first
    \(u\le v<M\). Then both deleted points lie below the common median, so
    \(\learning_c^{\mathrm{med}}(S_i\setminus U)
    =(x_{r+1}^{(i)}+x_{r+2}^{(i)})/2
    =M+(x_{r+2}^{(i)}-x_r^{(i)})/2\). Since
    \(-c\le x_r^{(i)}\le x_{r+2}^{(i)}\le c\), we get
    \(\learning_c^{\mathrm{med}}(S_i\setminus U)\in
    \bs{M,\min\bc{c,M+c}}\), an interval of length at most \(c\). The case
    \(M<u\le v\) is analogous.

    It remains to consider \(u\le M\le v\). Write the remaining points as
    \(y_1^{(i)}\le\cdots\le y_{2r-2}^{(i)}\), so that
    \(\learning_c^{\mathrm{med}}(S_i\setminus U)
    =(y_{r-1}^{(i)}+y_r^{(i)})/2\). If \(u<M<v\), one point is deleted from
    each side of the median, so \(y_{r-1}^{(i)}\le M\le y_r^{(i)}\). If a
    deleted value equals \(M\), then \(M\) is a sample value, hence
    \(x_r^{(i)}=x_{r+1}^{(i)}=M\), and the same bracketing still holds.
    Therefore
    \(\learning_c^{\mathrm{med}}(S_i\setminus U)\in
    \bs{(M-c)/2,(M+c)/2}\), again an interval of length \(c\). Thus, in all
    cases,
    \(\abs{\learning_c^{\mathrm{med}}(S_1\setminus U)
    -\learning_c^{\mathrm{med}}(S_2\setminus U)}\le c\).

    The bound is tight. For \(r\ge2\), take
    \(S_1=\bc{-c,\ldots,-c,0,\ldots,0}\), with \(r-1\) copies of \(-c\) and
    \(r+1\) copies of \(0\), and \(S_2=\bc{0,\ldots,0,c,\ldots,c}\), with
    \(r+1\) copies of \(0\) and \(r-1\) copies of \(c\). Both medians are \(0\).
    Deleting \(U=\bc{0,0}\) gives post-deletion medians \(-c/2\) and \(c/2\),
    so the separation is \(c\). Hence the largest admissible separation
    in a \((2,2,\Delta,\abs{\cdot};0,0)\)-s.d.e. is \(\Delta=c\).

    \item Case \(3\le m\le n-1\). Consider again the datasets from the previous
    case, and take \(U=\bc{0,0,0}\). Then
    \(\learning_c^{\mathrm{med}}(S_1\setminus U)=-c\) and
    \(\learning_c^{\mathrm{med}}(S_2\setminus U)=c\), so the separation is
    \(2c\). This achieves the trivial upper bound \(2c\). Since an s.d.e.
    allows \(|U|\le m\), this construction is admissible for every
    \(m\ge3\). Hence the largest admissible separation in a
    \((2,m,\Delta,\abs{\cdot};0,0)\)-s.d.e. is \(\Delta=2c\) for all
    \(3\le m\le n-1\).
\end{enumerate}

Now consider the odd case \(n=2r-1\), with \(r\in\mathbb N\). For each
\(i\in\{1,2\}\), the median is
\(M=\learning_c^{\mathrm{med}}(S_i)=x_r^{(i)}\).

\begin{enumerate}
    \item Case \(m=1\). Write \(U=\bc{u}\). If \(u<M\), then
    \(\learning_c^{\mathrm{med}}(S_i\setminus U)
    =(x_r^{(i)}+x_{r+1}^{(i)})/2=(M+x_{r+1}^{(i)})/2\), and hence
    \(\learning_c^{\mathrm{med}}(S_i\setminus U)\in\bs{M,(M+c)/2}\).
    Similarly, if \(u>M\), then
    \(\learning_c^{\mathrm{med}}(S_i\setminus U)\in\bs{(M-c)/2,M}\). Finally,
    if \(u=M\), then
    \(\learning_c^{\mathrm{med}}(S_i\setminus U)
    =(x_{r-1}^{(i)}+x_{r+1}^{(i)})/2\in\bs{(M-c)/2,(M+c)/2}\). In all cases,
    the two post-deletion medians lie in a common interval of length at most
    \(c\), so their separation is at most \(c\).

    This bound is tight. For \(r\ge2\), take
    \(S_1=\bc{-c,\ldots,-c,c,\ldots,c}\), with \(r-1\) copies of \(-c\) and
    \(r\) copies of \(c\), and \(S_2=\bc{c,\ldots,c}\), with \(2r-1\) copies
    of \(c\). Both medians are \(c\). Removing the common multiset \(U=\bc{c}\)
    gives post-deletion medians \(0\) and \(c\), so the separation is \(c\).
    Therefore the largest admissible separation in a
    \((2,1,\Delta,\abs{\cdot};0,0)\)-s.d.e. is \(\Delta=c\).

    \item Case \(2\le m\le n-1\). Take
    \(S_1=\bc{-c,\ldots,-c,0,\ldots,0}\), with \(r-1\) copies of \(-c\) and
    \(r\) copies of \(0\), and \(S_2=\bc{0,\ldots,0,c,\ldots,c}\), with \(r\)
    copies of \(0\) and \(r-1\) copies of \(c\). Both medians are \(0\).
    Deleting \(U=\bc{0,0}\) gives post-deletion medians \(-c\) and \(c\), so
    the separation is \(2c\), which is the trivial upper bound. Since an
    s.d.e. allows \(|U|\le m\), this construction is admissible for
    every \(2\le m\le n-1\). Hence the largest admissible separation in a
    \((2,m,\Delta,\abs{\cdot};0,0)\)-s.d.e. is \(\Delta=2c\) for all such \(m\).
\end{enumerate}
\end{proof}

\begin{proof}[Proof of \Cref{ex:svm}]
\label{app:proof-svm}
We write the construction first for odd \(n\). Let \(q=(n-1)/2\), and define
\[
\begin{aligned}
U_\rho
&=
\{(\rho e_1,+1)\},\\
S_{1,\rho}
&=
U_\rho
\biguplus_{j=1}^{q}
\{(-\rho(e_1+e_2),-1),(\rho(e_1+e_2),+1)\},\\
S_{2,\rho}
&=
U_\rho
\biguplus_{j=1}^{q}
\{(-\rho(e_1-e_2),-1),(\rho(e_1-e_2),+1)\},
\end{aligned}
\]
where \(\uplus\) denotes multiset union. Thus each dataset contains one copy of
\((\rho e_1,+1)\), together with \(q\) copies of each of the two remaining
labeled points, and therefore has size \(1+2q=n\). 
For even \(n\), the same construction is obtained by taking
\(q=(n-2)/2\) and adding one extra duplicate of any nondeleted point; duplicates
do not change the hard-margin SVM constraints. 
All points in
\(S_{1,\rho}\cup S_{2,\rho}\) lie in the unit ball because
\(\norm{\rho e_1}_2=\rho\le1\) and
\(\norm{\rho(e_1\pm e_2)}_2=\rho\sqrt{2}\le1\).

We first compute the SVM on the full datasets. For \(S_{1,\rho}\), the
margin constraints are
\[
\ip{w}{\rho e_1}\ge1,
\qquad
\ip{w}{\rho(e_1+e_2)}\ge1,
\]
or equivalently \(w_1\ge1/\rho\) and \(w_1+w_2\ge1/\rho\). The first
inequality implies \(\norm{w}_2\ge |w_1|\ge1/\rho\), while
\(\rho^{-1}e_1\) satisfies both inequalities with norm \(1/\rho\). Thus
\(\learning_d^{\mathrm{svm}}(S_{1,\rho})=\rho^{-1}e_1\). The same
argument for \(S_{2,\rho}\), whose nonredundant constraints are
\(w_1\ge1/\rho\) and \(w_1-w_2\ge1/\rho\), gives
\(\learning_d^{\mathrm{svm}}(S_{2,\rho})=\rho^{-1}e_1\).

After deleting \(U_\rho\), the first retained sample consists of the
two labeled points \((-\rho(e_1+e_2),-1)\) and
\((\rho(e_1+e_2),+1)\). Its hard-margin SVM is the minimum-norm vector
\(w\) such that \(\ip{w}{\rho(e_1+e_2)}\ge1\). By Cauchy--Schwarz, the
unique minimizer is parallel to \(e_1+e_2\) and satisfies the constraint
with equality, hence
\[
\learning_d^{\mathrm{svm}}(S_{1,\rho}\setminus U_\rho)
=
\frac{1}{2\rho}(e_1+e_2).
\]
Similarly,
\[
\learning_d^{\mathrm{svm}}(S_{2,\rho}\setminus U_\rho)
=
\frac{1}{2\rho}(e_1-e_2).
\]
Therefore
\[
\norm{
\learning_d^{\mathrm{svm}}(S_{1,\rho}\setminus U_\rho)
-
\learning_d^{\mathrm{svm}}(S_{2,\rho}\setminus U_\rho)
}_2
=
\frac{1}{\rho}.
\]
Since \(\rho>0\) can be taken arbitrarily small, for every \(\Delta>0\)
the construction gives a \((2,1,\Delta,\norm{\cdot}_2;0,0)\)-s.d.e. by choosing \(\rho\) sufficiently small. Since \(|U_\rho|=1\le m\),
the same construction gives a \((2,m,\Delta,\norm{\cdot}_2;0,0)\)-s.d.e.
for every \(1\le m\le n-1\).
\end{proof}

\section{Deferred Proofs from~\Cref{sec:memorisation}}

\begin{proof}[Proof of \Cref{pp:information-lb}]
    Since $M \indep (W^K, U^K) \;|\; S$, using the data processing inequality in $(a)$, we have that,
    \begin{gather}
    \label{eq:memorization-main-eq}
    \begin{aligned}
        I(M;S)
        & \overset{(a)}{\geq}
        I(M; W^K, U^K)
        =
        I \left(
            M; (W_{\pi(1)}, U_{\pi(1)}),
                \dots (W_{\pi(K)}, U_{\pi(K)})
        \right) \\
        & \overset{(b)}{=}
        \sum_{i=1}^K
        I \left(
            M; W_{\pi(i)}, U_{\pi(i)} | W_{\pi(<i)}, U_{\pi(<i)}
        \right) \\
        & \overset{(c)}{=}
        \sum_{i=1}^K \left[
            I \left(
                M; U_{\pi(i)} |  W_{\pi(<i)}, U_{\pi(<i)}
            \right)
            +
            I \left(
                M; W_{\pi(i)} |  W_{\pi(<i)}, U_{\pi(<i)}, U_{\pi(i)}
            \right)
        \right] \\
        & \overset{(d)}{\geq}
        \sum_{i=1}^K
        I \left(
            M; W_{\pi(i)} |  W_{\pi(<i)}, U_{\pi(<i)}, U_{\pi(i)}
        \right) \\
        & \overset{(e)}{\geq}
        \sum_{i=1}^K
        I \left(
            \overline{W}_{\pi(i)}; W_{\pi(i)} |  W_{\pi(<i)}, U_{\pi(<i)}, U_{\pi(i)}
        \right),
    \end{aligned}
    \end{gather}
    where $(b)$ and $(c)$ follow from the chain rule of mutual information, while $(d)$ follows from its non-negativity; $(e)$ uses the data processing inequality again: since \(\overline{W}_{\pi(i)}=\unlearning(M,U_{\pi(i)})\) is generated from \((M,U_{\pi(i)})\), it holds that
    \(
    \overline{W}_{\pi(i)} \indep W_{\pi(i)} \mid
        (M,  U_{\pi(i)}, W_{\pi(<i)}, U_{\pi(< i)})
    \).
\end{proof}

\begin{proof}[Proof of~\Cref{lem:t-recovery-bound}]

    Fix $s$ and $u$ with $u\subseteq s$ and $|u|\le m$. Throughout, write
    $\prob_{s,u}(\cdot) \coloneqq \prob(\cdot\mid S=s,U=u)$ and
    $p(w\mid s,u) \coloneqq \prob_{s,u}(W=w)$ for the law of the retraining target.
    
    Since $(\learning,\unlearning)$ is $(\alpha,\varepsilon)$-RU, the laws of
    $\overline W$ and $W^\emptyset$ under $\mathbb P_{s,u}$ satisfy
    $D_\alpha(\overline W\,\|\,W^\emptyset)\le\varepsilon$. For fixed
    $w\in\mathcal W$, applying~\Cref{lem:renyi-event-transfer} to the event
    $A_w\coloneqq\{x\in\mathcal W: x\ne w\}$ gives
    \begin{equation}\label{eq:lem42-transfer}
        \mathbb P_{s,u}\big(\overline W\ne w\big)
        \le e^{\gamma\varepsilon}\,
        \mathbb P_{s,u}\big(W^\emptyset\ne w\big)^{\gamma}.
    \end{equation}
    
    Multiplying~\eqref{eq:lem42-transfer} by $p(w\mid s,u)$, summing over $w$, and
    using Jensen's inequality, since $x\mapsto x^\gamma$ is concave for
    $0<\gamma<1$, gives
    \begin{equation}\label{eq:lem42-jensen}
        \sum_{w\in\mathcal W}p(w\mid s,u)\,\mathbb P_{s,u}\big(\overline W\ne w\big)
        \le e^{\gamma\varepsilon}\Big(\sum_{w\in\mathcal W}p(w\mid s,u)\,
        \mathbb P_{s,u}\big(W^\emptyset\ne w\big)\Big)^{\gamma}.
    \end{equation}
    Since $\overline W \indep W\mid
    (S,U)$, the left-hand side equals $\mathbb P_{s,u}(\overline W\ne W)$.
    
    We now bound the term inside the parentheses.
    Let $W_1,W_2$ be independent copies of $W$ under $\mathbb P_{s,u}$, that is,
    independent runs of $\learning(s\setminus u)$, and let
    $W^\emptyset_2 \coloneqq \unlearning(W_2,\emptyset)$. Expanding over the law of the
    target from which $W^\emptyset$ is generated,
    \begin{gather}
    \begin{aligned}\label{eq:lem42-coupling}
        \sum_{w\in\mathcal W}p(w\mid s,u)\,\mathbb P_{s,u}\big(W^\emptyset\ne w\big)
        &=\sum_{w,w'\in\mathcal W}p(w\mid s,u)\,p(w'\mid s,u)\,
        \mathbb P\big(W^\emptyset\ne w\mid W=w'\big) \\
        &=\mathbb P_{s,u}\big(W^\emptyset_2\ne W_1\big).
    \end{aligned}
    \end{gather}
    
    Since $\{W^\emptyset_2\ne W_1\}\subseteq\{W^\emptyset_2\ne W_2\}\cup\{W_1\ne W_2\}$,
    \begin{equation}\label{eq:lem42-split}
        \mathbb P_{s,u}\big(W^\emptyset_2\ne W_1\big)
        \le\mathbb P_{s,u}\big(W^\emptyset_2\ne W_2\big)
        +\mathbb P_{s,u}\big(W_1\ne W_2\big).
    \end{equation}
    
    It remains to bound the matched term.
    Since $W^\emptyset$ depends on $(S,U)$ only through $W$, $(d_0,\Gamma)$-utility
    gives
    \begin{equation}\label{eq:lem42-matched}
        \mathbb P_{s,u}\big(W^\emptyset_2\ne W_2\big)
        =\sum_{w\in\mathcal W}p(w\mid s,u)\,
        \mathbb P\big(\unlearning(w,\emptyset)\ne w\big)\le\eta .
    \end{equation}
    
    Combining (\ref{eq:lem42-jensen})--(\ref{eq:lem42-matched}),
    \[
        \mathbb P_{s,u}\big(\overline W\ne W\big)
        \le e^{\gamma\varepsilon}\big(\eta+\mathbb P_{s,u}(W_1\ne W_2)\big)^{\gamma}
        \le e^{\gamma\varepsilon}\big(\eta+\Delta_{\learning}\big)^{\gamma}.
    \]
    Taking the supremum over $(s,u)$, and using $\beta_{\mathrm{un}}\le1$, gives
    $\beta_{\mathrm{un}}\le\bar\beta$. The two error sources are visible in the
    bound: $\eta$ accounts for the empty-request map moving the target,
    $\Delta_{\learning}$ for the randomness of the learner, and the factor
    $e^{\gamma\varepsilon}(\cdot)^{\gamma}$ for approximate unlearning. In
    particular, for deterministic $\learning$, $\Delta_{\learning}=0$ and
    $\beta_{\mathrm{un}}\le e^{\gamma\varepsilon}\eta^{\gamma}$.
\end{proof}

\begin{proof}[Proof of~\Cref{thm:information-lb-discrete}]
    To prove the memorization lower bound, we start from the result of~\Cref{pp:information-lb} and will bound the term,
    \[
        I \left(
            W_{\pi(i)}; \overline{W}_{\pi(i)} \mid
            W_{\pi(<i)}, U_{\pi(\le i)}
        \right).
    \]
    Assuming the model space \( \mathcal{W} \) is discrete and finite, we have that,
    \begin{gather}
    \begin{aligned}
        I \left(
            W_{\pi(i)}; \overline{W}_{\pi(i)} \mid
            W_{\pi(<i)}, U_{\pi(\le i)}
        \right)
        &=
        H \left(
            W_{\pi(i)} |  W_{\pi(<i)}, U_{\pi(\le i)}
        \right)
        -
        H \left(
            W_{\pi(i)} |  \overline{W}_{\pi(i)}, W_{\pi(<i)}, U_{\pi(\le i)}
        \right) \\
        &\overset{(a)}{\geq}
        H \left(
            W_{\pi(i)} |  W_{\pi(<i)}, U_{\pi(\le i)}
        \right)
        -
        H \left(
            W_{\pi(i)} |  \overline{W}_{\pi(i)}
        \right),
    \end{aligned}
    \end{gather}
    where $(a)$ uses that conditioning reduces entropy.

    By Fano's inequality applied to the pair
    \[
        \left(
            W_{\pi(i)},\,
            \overline{W}_{\pi(i)}
        \right),
    \]
    we get
    \[
        H\left(
            W_{\pi(i)}
            \mid
            \overline{W}_{\pi(i)}
        \right)
        \leq
        h(P_{e,\pi(i)})
        +
        P_{e,\pi(i)}\log(|\mathcal W|-1)
        =:
        F(P_{e,\pi(i)}) %
        ,
    \]
    where
    \[
        P_{e,\pi(i)}
        \coloneqq
        \prob{
            \overline{W}_{\pi(i)}\neq W_{\pi(i)}
        }.
    \]
    Moreover,
    \[
        P_{e,\pi(i)}
        =
        \expp_{S,U_{\pi(i)}}
        \left[
        \prob{
            \overline{W}_{\pi(i)}\neq W_{\pi(i)}
            \mid S,U_{\pi(i)}
        }
        \right]
        \leq
        \expp_{S,U_{\pi(i)}}[\beta_{\mathrm{un}}]
        =
        \beta_{\mathrm{un}}
        \leq
        \bar\beta .
    \]
    The function $F$ is increasing on $[0, 1 - \frac{1}{|\mathcal W|}]$, since
    $F'(p) = \log\frac{(1-p)(|\mathcal W|-1)}{p} \geq 0$ there, and it attains its maximum
    $\log|\mathcal W|$ at $p = 1 - \frac{1}{|\mathcal W|}$. Hence, with
    $\tilde\beta \coloneqq \min\{\bar\beta,\, 1 - \frac{1}{|\mathcal W|}\}$, we have
    $F(P_{e,\pi(i)}) \leq F(\tilde\beta)$: by monotonicity if $\bar\beta \leq 1 - \frac{1}{|\mathcal W|}$,
    and by maximality otherwise. Hence, 
    \[
        H\left(
            W_{\pi(i)}
            \mid
            \overline{W}_{\pi(i)}
        \right)
        \le
        h(\tilde\beta) + \tilde\beta\log(|\mathcal W|-1). %
    \]
    Therefore,
    \[
        I\left(
            W_{\pi(i)};\overline{W}_{\pi(i)}
            \mid
            W_{\pi(<i)},U_{\pi(\le i)}
        \right)
        \ge
        H\left(
            W_{\pi(i)}
            \mid
            W_{\pi(<i)},U_{\pi(\le i)}
        \right)
        -
        h(\tilde\beta) %
        -
        \tilde\beta\log(|\mathcal W|-1). %
    \]
    Summing over \(i=1,\ldots,K\) and using~\Cref{pp:information-lb} gives the claim.
    
\end{proof}

\subsection{Deferred Proofs from~\Cref{sec:examples-disc}}
\label{app:example-proofs}
\begin{proof}[Proof of~\Cref{pp:canonical-thresholds}]
    Partition \(\bs{N}\) into \(n\) consecutive bins
    \[
        B_j=\bc{(j-1)q+1,\ldots,jq},
        \qquad j\in\bs{n},
    \]
    where \(q=N/n\). Draw \(x_j\sim\mathrm{Unif}(B_j)\) independently, and set
    \[
        y_j=\dsI\bc{j\ge m}.
    \]
    Let
    \[
        S=((x_1,y_1),\ldots,(x_n,y_n)).
    \]
    Since the bins are ordered, every negative example lies to the left of
    every positive example, so \(S\) is realizable by a threshold. By the
    canonical tie-breaking rule, the ERM output on the full sample is
    \[
        \learning_{\mathrm{ERM}}(S)=h_{x_m}.
    \]

    For each \(i\in\bs{m}\), define the deletion request
    \[
        U_i
        =
        \bc{(x_j,0):1\le j\le m-i}
        \cup
        \bc{(x_j,1):m\le j\le m+i-1}.
    \]
    The first set has size \(m-i\), the second has size \(i\), and hence
    \(\abs{U_i}=m\). Since \(n\ge2m\), the point \(x_{m+i}\) exists. After deleting
    \(U_i\), the leftmost remaining positive example is \(x_{m+i}\). Therefore
    \[
        W_i
        \coloneqq
        \learning_{\mathrm{ERM}}(S\setminus U_i)
        =
        h_{x_{m+i}}.
    \]

    Let \(H_i=(W_{<i},U_{\le i})\). The variables in \(H_i\) reveal only points
    from bins other than \(B_{m+i}\). In particular, the previous retraining
    targets \(W_1,\ldots,W_{i-1}\) reveal
    \(x_{m+1},\ldots,x_{m+i-1}\), all of which are already deleted by \(U_i\).
    Since the \(x_j\)'s are independent across bins,
    \[
        x_{m+i}\mid H_i\sim\mathrm{Unif}(B_{m+i}).
    \]
    Distinct threshold locations define distinct hypotheses, so
    \[
        H(W_i\mid H_i)=\log q.
    \]

    Applying~\Cref{thm:information-lb-discrete} and substituting
    \(q\) for each \(i\in\bs{m}\), gives the result. Finally, for the
    bare ERM output, we have \(M\), and
    \(x_m\sim\mathrm{Unif}(B_m)\). Hence
    \[
        I(M;S)_{\mathrm{bare}}=H(M)=\log q.
    \]
\end{proof}

\begin{proof}[Proof of \Cref{prop:coordinate-pca}]
    Partition \(\bs{d}\) into \(n\) disjoint blocks
    \[
        B_j=\bc{(j-1)q+1,\ldots,jq},
        \qquad j\in\bs{n}.
    \]
    Let \(e_a\in\reals^d\) denote the \(a\)-th standard basis vector.
    Fix strictly decreasing positive numbers
    \[
        \lambda_1>\lambda_2>\cdots>\lambda_n>0.
    \]
    Draw \(a_j\sim\mathrm{Unif}(B_j)\) independently and set
    \[
        z_j=\sqrt{\lambda_j}\,e_{a_j}.
    \]
    Let \(S=(z_1,\ldots,z_n)\).

    The learner returns
    \[
        \learning_{\mathrm{PCA}}(D)
        =
        e_{\widehat a(D)},
        \qquad
        \widehat a(D)
        \in
        \argmax_{a\in\bs{d}}
        \sum_{z\in D}\ip{e_a}{z}^2,
    \]
    with an arbitrary fixed tie-breaking rule. On the full sample, because
    the \(\lambda_j\)'s are strictly decreasing,
    \[
        M_{\mathrm{bare}}
        =
        \learning_{\mathrm{PCA}}(S)
        =
        e_{a_1}.
    \]

    For \(i\in\bs{m}\), define
    \[
        U_i
        =
        \bc{z_1,\ldots,z_i}
        \cup
        \bc{z_{n-m+i+1},\ldots,z_n},
    \]
    where the second set is empty when \(i=m\). The first set has size \(i\)
    and the second has size \(m-i\), so \(\abs{U_i}=m\). The assumption
    \(2m\le n\) ensures that these deletions do not remove \(z_{i+1}\).

    After deleting \(U_i\), the largest remaining empirical second moment is
    \(\lambda_{i+1}\), attained uniquely at coordinate \(a_{i+1}\). Hence
    \[
        W_i
        \coloneqq
        \learning_{\mathrm{PCA}}(S\setminus U_i)
        =
        e_{a_{i+1}}.
    \]

    Let \(H_i=(W_{<i},U_{\le i})\). The variables in \(H_i\) reveal
    \(a_1,\ldots,a_i\), possibly some tail coordinates, and
    \(a_2,\ldots,a_i\) through the previous retraining targets. They reveal no
    information about \(a_{i+1}\). Since the blocks are independent,
    \[
        a_{i+1}\mid H_i\sim\mathrm{Unif}(B_{i+1}).
    \]
    Therefore
    \[
        H(W_i\mid H_i)=\log q.
    \]

    Applying the conditional-support Fano argument from
    \Cref{sec:examples-disc}, with \(q\) possible targets per request,
    gives the result. For the bare
    PCA output, \(M=e_{a_1}\), so
    \[
        I(M;S)_{\mathrm{bare}}=H(e_{a_1})=\log q.
    \]
\end{proof}

\begin{proof}[Proof of \Cref{prop:ridge-regression}]
    The main idea is that, for one-hot inputs, \(\ip{w}{e_a}=w_a\)
    for every \(a\in[d]\). Therefore, the ridge objective separates
    into independent one-dimensional problems. For a dataset \(D\),
    its unique minimizer satisfies
    \[\bs{\learning_{\mathrm{ridge}}(D)}_a=T_a(D)/(N_a(D)+\lambda),\]
    where \(N_a(D)\) is the number of examples with input \(e_a\),
    and \(T_a(D)\) is the sum of their labels. We construct datasets
    with equally many copies of \((e_a,+1)\) and \((e_a,-1)\), so
    all fitted coefficients are zero regardless of \(N_a(D)\).
    Deleting one copy of \((e_a,+1)\) changes the corresponding label
    sum to \(-1\), making the retrained coefficient depend on
    \(N_a(D)\). Accurate unlearning therefore requires retaining
    information about these counts. The remaining examples in the
    construction ensure that the total dataset size stays fixed
    despite this randomness.

    Draw \(C_1,\ldots,C_b\) independently and uniformly from
    \(\bc{1,2,3}\). Block \(j\) contains \(C_j\) copies of each of
    \((e_{2j-1},+1)\) and \((e_{2j-1},-1)\), and \(4-C_j\) copies
    of each of \((e_{2j},+1)\) and \((e_{2j},-1)\). Each block
    therefore has eight examples, giving \(n=8b\). Place one copy
    of \((e_{2j-1},+1)\) first in each block, order the remaining
    examples deterministically, and concatenate the blocks to obtain
    \(S=(z_1,\ldots,z_n)\). Since all label sums are zero,
    \(\learning_{\mathrm{ridge}}(S)=0\).

    For each \(i\in[b]\), let
    \(U_i=\bc{z_{8(i-1)+1}}=\bc{(e_{2i-1},+1)}\). Both its index
    and value are fixed, so this singleton request reveals nothing
    about \(C_i\). After deletion, there are \(2C_i-1\) examples
    with input \(e_{2i-1}\), with label sum \(-1\), while all other
    label sums remain zero. Thus
    \[
        W_i=\learning_{\mathrm{ridge}}(S\setminus U_i)
        =-\frac{1}{2C_i-1+\lambda}\,e_{2i-1}.
    \]
    The three values of \(C_i\) give three distinct retraining
    models it needs to match. Let \(\mathcal H_i=(W_{<i},U_{\le i})\).
    Previous targets reveal only \(C_1,\ldots,C_{i-1}\), and the
    requests are deterministic. Hence \(C_i\) is still uniformly
    random given \(\mathcal H_i\), so
    \(H(W_i\mid\mathcal H_i)=\log3\). For \(\beta_3<2/3\),
    applying the conditional-support Fano bound from
    \Cref{sec:examples-disc} and \Cref{pp:information-lb} gives
    \[
        I(M;S)
        \ge b\bigl[\log3-h(\beta_3)-\beta_3\log2\bigr]
        =\frac n8\bigl[\log3-h(\beta_3)-\beta_3\log2\bigr].
    \]
    For \(\beta_3=2/3\), the bound follows from nonnegativity.
    Finally, the bare fitted vector is identically zero on this
    distribution, so \(I(M;S)_{\mathrm{bare}}=0\).
\end{proof}

\begin{proof}[Proof of \Cref{prop:factorized-matrix-completion}]
    Set \(r=r_\star\). By definition of \(r_\star\), we have
    \(r^2\le d\) and \(r(m+1)\le n\). Since \(d>4\) and \(n>3m\), we have
    \(r\ge2\).

    We use the first \(r^2\) ambient coordinates as an \(r\times r\)
    matrix.  If \(n>r(m+1)\), we add \(n-r(m+1)\) deterministic dummy
    measurements of the form \((1,0_r,0)\). These impose no constraint
    and do not affect the fitted matrix.

    Draw a uniformly random permutation \(\sigma\) of \(\bs{r}\). For
    each row \(i\in\bs{r}\), include \(m\) indexed copies of the row-sum
    measurement
    \[
        c_i^{(a)}=(i,\mathbf 1_r,r),
        \qquad a\in\bs{m},
    \]
    and one coordinate measurement
    \[
        z_i=(i,e_{\sigma(i)},1).
    \]
    Thus the non-dummy part of the dataset has size \(r(m+1)\). The
    copies are treated as distinct indexed examples, so a deletion
    request may delete all \(m\) row-sum copies in a row.

    We first compute the fitted matrix on the full dataset. Because
    the data are consistent and \(\lambda\to0\), the learner selects,
    among interpolating matrices, one minimizing
    \[
        \sum_{i=1}^r \norm{x_i}_2.
    \]
    The rows decouple. Fix row \(i\). The constraints are
    \[
        \sum_{\ell=1}^r (x_i)_\ell=r,
        \qquad
        (x_i)_{\sigma(i)}=1.
    \]
    Therefore the remaining \(r-1\) coordinates must sum to \(r-1\). By
    Cauchy--Schwarz,
    \[
        \sum_{\ell\neq\sigma(i)} (x_i)_\ell^2
        \ge
        \frac{(r-1)^2}{r-1}
        =
        r-1,
    \]
    with equality iff all remaining coordinates are equal to \(1\). Hence the
    unique minimum-norm interpolating row is
    \[
        x_i=\mathbf 1_r.
    \]
    This holds for every row, so the full fitted matrix is deterministic:
    \[
        M_{\mathrm{bare}}
        =
        \learning_{\mathrm{fac}}(S)
        =
        \mathbf 1_r\mathbf 1_r^\top.
    \]

    For each row \(i\in\bs{r}\), define
    \[
        U_i=\bc{c_i^{(1)},\ldots,c_i^{(m)}}.
    \]
    This deletion request has size \(m\). After deleting \(U_i\), row \(i\)
    has only the coordinate constraint
    \[
        (x_i)_{\sigma(i)}=1.
    \]
    The unique minimum-\(\ell_2\)-norm row satisfying this constraint is
    \[
        x_i=e_{\sigma(i)}.
    \]
    All other rows still have both their row-sum and coordinate measurements,
    so they remain equal to \(\mathbf 1_r\). Thus
    \(W_i=\learning_{\mathrm{fac}}(S\setminus U_i)\) is the matrix whose
    \(i\)-th row is \(e_{\sigma(i)}\) and whose other rows are
    \(\mathbf 1_r\). In particular, \(W_i\) reveals exactly \(\sigma(i)\).

    Let $\mathcal H_i = (W_{<i}, U_{\le i})$. The deletion requests themselves are deterministic
    functions of the row index and reveal no information about $\sigma$. After observing
    $W_1, \dots, W_{i-1}$, one knows $\sigma(1), \dots, \sigma(i-1)$. Since $\sigma$ is a uniformly
    random permutation, conditional on this information, $\sigma(i)$ is uniform over the remaining
    $q_i \coloneqq r - i + 1$ columns. Therefore
    \[
        H(W_i \mid \mathcal H_i) = \log q_i .
    \]
    Applying the conditional-support Fano argument from \Cref{sec:examples-disc} with $q_i$ possible
    targets for request $i$, and noting that the last request ($q_r = 1$) has zero conditional entropy
    and zero penalty, we get, with $s = q_i$,
    \[
        I(M;S)
        \geq
        \log(r!) - \sum_{s=2}^{r} \bs{h(\beta_s) + \beta_s \log(s-1)},
    \]
    where we used $\sum_{s=2}^{r} \log s = \log(r!)$.
    Finally, $M$ is deterministic on this distribution, so
    $I(M;S)_{\mathrm{bare}} = 0$.
    
\end{proof}

\end{document}